\documentclass{article}

 \usepackage[main, final]{neurips_2026}
\usepackage[utf8]{inputenc} % allow utf-8 input
\usepackage[T1]{fontenc}    % use 8-bit T1 fonts
\usepackage{hyperref}       % hyperlinks
\usepackage{url}            % simple URL typesetting
\usepackage{booktabs}       % professional-quality tables
\usepackage{amsfonts}       % blackboard math symbols
\usepackage{nicefrac}       % compact symbols for 1/2, etc.
\usepackage{microtype}      % microtypography
\usepackage{xcolor}         % colors

\usepackage{algorithm}
\usepackage{algorithmicx}
\usepackage{algpseudocode}
\usepackage{amsmath}
\usepackage{amssymb}
\usepackage{float}
\usepackage{amsthm}
\usepackage{graphicx}
 \usepackage{multirow} 
 \usepackage{makecell}
 \usepackage{subcaption} 
\newtheorem{theorem}{Theorem}[section] % Theorem 环境，按章节编号
\newtheorem{corollary}{Corollary}[theorem]  % 编号跟着 theorem 走
\newtheorem{proposition}[theorem]{Proposition}
\newtheorem{lemma}[theorem]{Lemma}     % Lemma 环境，和 Theorem 共享编号
\newtheorem*{proposition*}{Proposition}

\theoremstyle{definition}
\theoremstyle{plain}        % 保留斜体正文
\newtheorem*{remark}{Remark}

\title{Geometry-Aware Online Scheduling for LLM Serving: From Theoretical Bound to System Practice}

\author{%
    Li Kong\\
  Gaoling School of Artificial Intelligence\\
  Renmin University of China\\
  Beijing, China \\
  \texttt{kongli@ruc.edu.cn} \\
  \And
  Qi Qi\thanks{Corresponding author.} \thanks{Also with, Beijing Key Laboratory of Research on Large Models and Intelligent Governance. Also with, Engineering Research Center of Next-Generation Intelligent Search and Recommendation, MOE.} \\
  Gaoling School of Artificial Intelligence\\
  Renmin University of China\\
  Beijing, China \\
  \texttt{qi.qi@ruc.edu.cn} \\
  \AND
  Yinyu Ye \\
  Department of Industrial Engineering and Decision Analytics \\
  HKUST \\
  Hong Kong, China \\
  \texttt{yyye@ust.hk} \\
  \And
  Zijie Zhou \\
  Department of Industrial Engineering and Decision Analytics \\
  HKUST \\
  Hong Kong, China \\
  \texttt{jerryzhou@ust.hk} \\
}

\begin{document}

\maketitle

%%%%%%%%%%%%%%%%%%%%%%%%%%%%%%%%%%%%%%%%%%%%%%%%%%%%%%%%%%%%
\begin{abstract}

The rapid growth of interactive Large Language Model serving has made efficient management of dynamic Key-Value cache footprints increasingly important for inference performance. Modern inference systems overwhelmingly rely on time-centric scheduling heuristics, such as Shortest Job First. However, their classical guarantees are rooted in traditional scheduling modeling, failing to capture the highly dynamic, 2D spatio-temporal geometric growth specific to LLM inference mechanisms. To resolve this, we propose the geometry-aware online scheduling by introducing the Smallest Volume First (SVF) algorithm and its highly efficient variant, 1-bit SVF.
Via a novel volume-certificate proof, we establish a worst-case approximation ratio for SVF that approaches \textbf{3} in the high-concurrency regime of LLM serving, substantially improving upon the prior best constant of 48. Building upon this core breakthrough, we complete a comprehensive theoretical taxonomy analyzing our algorithms across different traffic scenarios and information availability. Practically, we seamlessly integrate our approach as a plug-and-play layer in vLLM. Extensive evaluations on Llama-3.1 models demonstrate comprehensive performance gains: SVF delivers strong reductions in both average and tail latency, while 1-bit SVF, with merely a single bit of information, achieves competitive throughput and latency. This work establishes a theoretically sound and empirically proven approach for resolving memory-constrained scheduling in modern LLM deployments. To facilitate future research, our code is available at \url{https://github.com/Li-Beverly-Kong/Geometry-Aware-Online-Scheduling.git}.

\end{abstract}

\section{Introduction}
% intro先起个初稿，要和related work相辅相成。 有好些reference得等related work整理好后再加、再完善----5.1要完成这两部分
% 0503 和子杰老师讨论后： 
% 1. intro不用对调度算法做任何建模上的限制，但在section 4建模部分要明确说明模型是memory-bound（不要对建模overclaim）
% 2. intro或者在method部分要建立volume=geometry-aware的概念-Done
% 4. appendix做下discussion，看看我们的算法能和什么类型的工作兼容吧（preflix caching应该不兼容，所以正交不成立）
% 5. 多强调自己的理论价值，单scheduler的话workload是不够的！！！
% TODO:后期回来修改第三点contribution，把改进倍数这些写上去（每篇文章都这么写！！！！！）

% TODO：前两段需要精简一点-更精辟到位的表达。
Large Language Models (LLMs) are now a foundational component of modern Internet services, powering applications that serve millions of concurrent users~\cite{li2024llm, zhen2025taming}.
This evolution from experimental AI milestones to ubiquitous cognitive engines is evident in applications ranging from interactive AI agents and real-time coding assistants to dynamic search engines (e.g., \cite{achiam2023gpt, xiong2024search, liu2024deepseek, murali2024ai}). 
As LLMs become deeply embedded in consumer-facing scenarios, the primary bottleneck in deployment has decisively shifted from pure model capability to inference efficiency~\cite{zhen2025taming}. 
In these interactive paradigms, latency is no longer just a system metric, but rather the determinant of user experience. Excessive latency translates to broken conversational flows, diminished engagement, and even product abandonment. Consequently, minimizing the end-to-end latency perceived by each individual user has emerged as the foremost optimization objective for modern LLM deployment frameworks.

To mitigate the high latencies inherent in heavy workloads, the community has explored various scheduling disciplines. Early implementations relying on First-Come-First-Serve (FCFS)~\cite{yu2022orca, kwon2023efficient} policies often suffer from significant head-of-line blocking, where shorter requests are trapped behind long-running tasks. To resolve this, recent research has shifted towards prediction-based scheduling to forecast request characteristics and enable prioritization.
In classical scheduling theory, Shortest Remaining Processing Time (SRPT) is proven to minimize average latency. However, in LLM inference, preemption incurs prohibitive overheads such as Key-Value (KV) cache eviction and subsequent recomputation. Since modern inference engines manage these complex preemptive mechanics at the system execution level, the focus of scheduling research has naturally shifted towards optimizing the global priority metric, with Shortest Job First (SJF) emerging as the key heuristic.
\textit{A comprehensive discussion of related LLM scheduling literature is provided in Appendix \ref{related_work}.}
%that dictates admittance and execution order. This architectural decoupling has naturally led to the widespread adoption of Shortest Job First (SJF) as the foundational heuristic to guide the prioritization. 
%However, this collective pursuit harbors a critical blind spot: the community implicitly treats SJF as a universal gold standard, failing to recognize that this legacy time-centric heuristic is inherently flawed under the unique geometric resource dynamics of LLM inference.

% --- Paragraph 3: Theoretical Collapse of SJF ---
However, this widespread focus on SJF masks a fundamental mismatch: as a legacy heuristic, SJF is ill-suited for the unique geometric dynamics of LLM inference.
First, autoregressive decoding couples execution time with memory expansion; each generated token increases the KV cache footprint. Second, the widespread adoption of continuous batching creates dynamic fluctuations in concurrent batch sizes. Recently, a foundational KV cache-centric model \cite{jaillet2025online} successfully formalized the spatio-temporal interplay inherent to LLM serving. Under this framework, \cite{wang2025llm} showed that SJF yields an \textbf{unbounded} worst-case approximation ratio, exposing a fundamental gap between classical scheduling theory and the geometric realities of LLM inference.

Crucially, resolving this fundamental gap requires moving beyond empirical trial-and-error to a rigorous theoretical foundation. We anchor our approach in formal mathematical modeling for two vital reasons. First, theoretical bounds provide universal, structural guarantees, ensuring that system performance is robust rather than merely overfitted to specific benchmarks~\cite{zhou2026positionllmservingneeds}. Second, formal theoretical derivation inspires the discovery of novel algorithmic structures, elevating the design process from ad-hoc engineering heuristics to principled geometric solutions. 

% --- Paragraph 4: Our Algorithm Design (SVF & 1-bit SVF) ---
Driven by these theoretical insights, we propose a \textit{geometry-aware} online scheduling approach, instantiated through two complementary algorithms: Smallest Volume First (SVF) and \textit{1-bit SVF}. Instead of isolating execution time, SVF natively evaluates the integral of memory consumption over time, redefining request priority based on the predicted KV cache volume. Concurrently, 1-bit SVF serves as a highly lightweight alternative. It demonstrates that extracting a single bit of information is sufficient to achieve effective scheduling.
In summary, our primary contributions are threefold:

%From a deployment perspective, our approach serves as a plug-and-play scheduler that can be readily integrated with existing execution-level optimizations to further elevate the overall serving efficiency.
% Ultimately, by uniting rigorous mathematical guarantees with state-of-the-art empirical performance, our geometry-aware approach successfully closes the gap between scheduling theory and the physical realities of modern LLM deployment.

% --- Paragraph 5: Theoretical Guarantees & Empirical Results ---改为先high level 再分三点展开

% --- Paragraph 6: Contributions ---
\begin{itemize}
    \item \textbf{Geometry-Aware Online Scheduling Algorithms:} To tackle the unique spatio-temporal properties of LLM inference, we propose SVF and its lightweight variant 1-bit SVF. These algorithms realize a fundamental shift from legacy 1D time-centric heuristics to 2D geometry-aware prioritization, achieving highly effective scheduling in online settings.

    \item \textbf{Comprehensive Theoretical Guarantees:} Based on the foundational modeling proposed by \cite{jaillet2025online}, we complete a comprehensive theoretical taxonomy analyzing our algorithms across different traffic scenarios and information availability. Most notably, under burst arrivals, our novel volume-certificate proof sharpens SVF's \textbf{worst-case approximation ratio towards 3} in the high-concurrency regime that characterizes LLM serving--improving the prior best constant of 48. This result, along with our novel proof methodology, represents a theoretical advancement over prior LLM scheduling literature.
    
    \item \textbf{Practical Implementation and Validated Gains:} From a deployment perspective, our approach serves as a plug-and-play scheduler that can be integrated with existing execution-level optimizations to further elevate the overall serving efficiency. Seamlessly implemented within vLLM, we comprehensively validate their superiority through extensive evaluations, including baseline comparisons, ablation studies, and system overhead profiling. The results confirm complementary advantages of our algorithms: SVF delivers strong reductions in both average and tail latency, whereas 1-bit SVF, requiring only minimal information, achieves competitive throughput and latency.%These evaluations indicate that our geometry-aware algorithms enhance efficiency by reducing average latency and improve fairness by lowering tail latency.

\end{itemize}
% TODO：后期回来修改第三点contribution，把改进倍数这些写上去（每篇文章都这么写！！！！！）

\begin{figure}[t]
    \centering
    \includegraphics[width=0.8\linewidth]{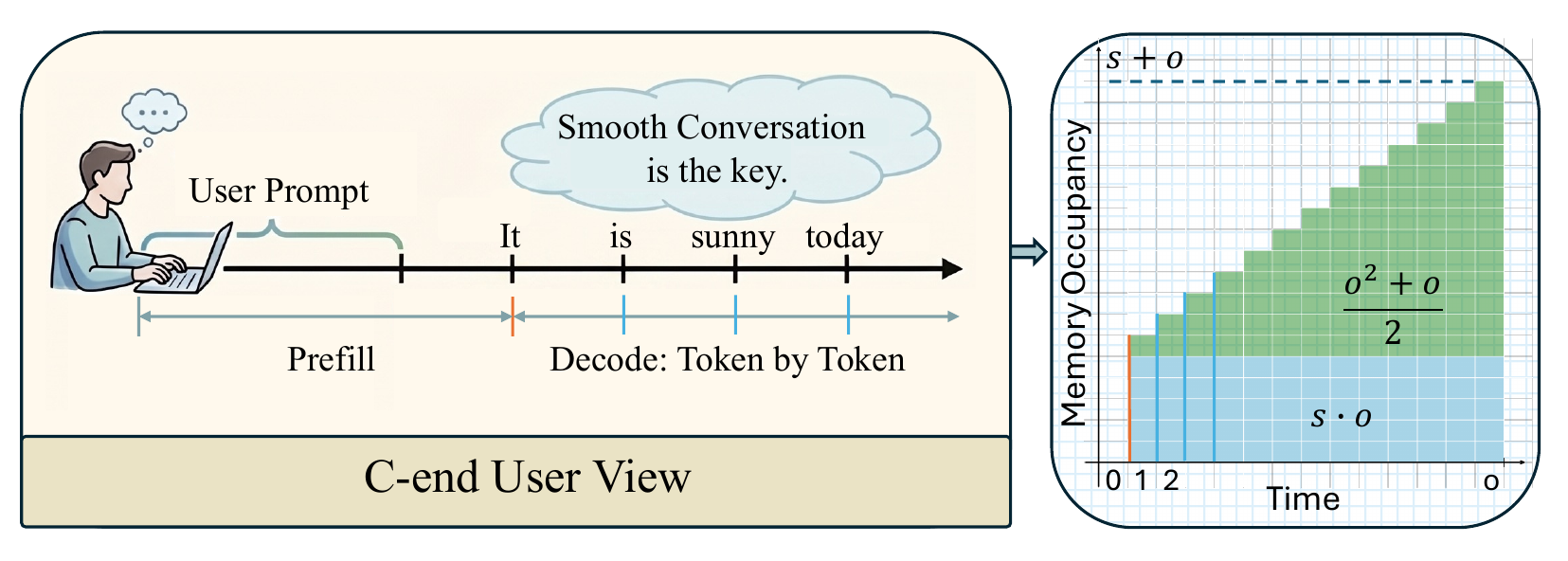}
    \caption{LLM Inference Process and the Calculation of Volume.}
    \label{fig:preliminaries}
\end{figure}

\section{Method}

\subsection{Preliminaries}
% 貌似这一块也还是在迎合、展开说两个challenges。然后提出我们的方法。所以和3.2一开始的衔接也要做好。
\textbf{LLM inference}.
The generation process of LLMs is typically partitioned into two distinct phases: \textit{prefill} and \textit{decode}. As shown in Figure \ref{fig:preliminaries}, during the prefill phase, the system processes the entire input prompt (of length $s$), generating the first token and allocating the initial KV cache. Subsequently, in the decode phase, the model generates the remaining tokens autoregressively, where each new token expands the KV cache. This token-by-token expansion means a request's memory footprint is not static, but grows over time until the request completes (with a total output length $o$).

\textbf{Continuous Batching}.
Traditional static batching suffers from severe fragmentation due to the highly variable output lengths of LLM requests. Continuous batching resolves this by operating at the token level. It dynamically evicts completed requests and injects new requests from the waiting queue into the active batch after any decoding step. While significantly maximizing GPU compute utilization, this mechanism also introduces a form of dynamism different from traditional scheduling: the system's concurrent batch size constantly fluctuates after each decoding step.

\subsection{Algorithms}
% 解释伪代码的的Event2，强调我们并未修改vLLM底层复杂的Paged Attention和KV Cache管理逻辑。我们的算法可以视为一个plug-and-play scheduler，这大大提升了算法的工业可用性。

\textbf{Smallest Volume First (SVF) Algorithm}.
Motivated by the unique spatio-temporal dynamics of LLM inference, we propose the Smallest Volume First scheduling algorithm, explicitly evaluating the 2D memory footprint of a request over its entire lifecycle. As illustrated in the right panel of Figure \ref{fig:preliminaries}, given a request $r_i$ with prompt length $s_i$ and known output length $o_i$, we can quantify this dynamic memory occupancy by its \textit{Volume} $v_i$:
\begin{equation}
    v_i = s_i \cdot o_i + \frac{o_i^2 + o_i}{2}
\end{equation}

By equating the scheduling priority directly with this geometry-aware volume, SVF acts as a strictly greedy policy that sorts requests in ascending order of $v_i$. As detailed in Algorithm \ref{alg:svf}, the practical implementation decouples this logic into two processes to eliminate blocking overhead on the critical execution path. Specifically, \textbf{Process 1} operates as a plug-and-play scheduler that queries the output length predictor $\mathcal{P}$ in a batched manner. Concurrently, \textbf{Process 2} continuously pops requests with the highest priority from $Q_{wait}$ and admits them into the running batch $Q_{run}$, maximizing system concurrency until the KV cache capacity is fully utilized.

% 防御性描述：解释采用greedy算法而非online learning的3点理由，之前就准备好的——那么后续section4的burst和poisson就为两个场景下均有理论保障，进一步维护了系统稳定性
% TODO: 感觉这部分可能还有改进空间。或者看看这里会不会引发什么攻击点。
We deliberately design SVF as an explicit metric-based greedy algorithm over online learning due to the inherent non-stationarity of LLM traffic and the complex states of 2D memory packing. As established in Section \ref{sec:theory}, SVF provides robust theoretical guarantees across diverse traffic patterns. A detailed rationale for favoring greedy metrics is provided in Appendix \ref{discussion:greedy}.

\textbf{1-bit Smallest Volume First (1-bit SVF) Algorithm}. % 轻量级版本，而不是基于SVF的改进。不能拉踩
As a lightweight alternative to the full-prediction model, we introduce 1-bit SVF, demonstrating that extracting merely \textit{1 extra bit} of information is sufficient for effective scheduling. 
Instead of predicting the exact generation length, this variant utilizes an extremely low-overhead binary classifier to categorize an incoming request as either ``short'' (class $m=0$) or ``long'' (class $m=1$). The system then assigns a theoretically derived global proxy length $O_m$ to calculate the \textit{Proxy Predicted Volume} as:
\begin{equation}
    \hat{v}_i = s_i \cdot O_m + \frac{1}{2}(O_m^2 + O_m).
\end{equation}
By relying on a simple classification boundary rather than exact numerical estimation, this 1-bit design significantly minimizes computational overhead while preserving the core spatio-temporal scheduling guarantees. The complete pseudocode is provided in Algorithm \ref{alg:onebit}.

Our analysis examines two complementary realizations of the same geometry-aware scheduling principle. Section~\ref{sec:theory} studies idealized information regimes, following the problem formulation of \cite{jaillet2025online}, to isolate the structural benefit of volume-based prioritization. Section~\ref{sec:experiments} evaluates its deployable instantiation, where output information is obtained from lightweight predictors.
\section{Theoretical Guarantees}
\label{sec:theory}
In this section, we provide a mathematical analysis of our geometry-aware scheduling paradigm. We first introduce the formal system model that abstracts the spatio-temporal dynamics of LLM serving, followed by theoretical analyses under varying arrival patterns. We operate under the standard theoretical assumption that the output generation length of each request is known.

\subsection{Problem Formulation}
% 建模的合理性
% - 合理的对于prefill和decode的cache成本的assumption
% - 充分反映continuous batching机制
% - 考虑多种因素（比如通信成本的难以量化），决定以不可抢占的形式进行
% - 【Anything else? 感觉这部分防御性描述得做够】
We model the online scheduling problem following the memory-constrained framework of \cite{jaillet2025online}. The system operates in discrete time steps with a global memory capacity $M$. An instance comprises $n$ requests, where each request $i \in [n]$ is characterized by its arrival time $a_i$, prompt length $s_i$, and output generation length $o_i$. 
To formulate the optimal non-preemptive schedule as an Integer Program, we introduce a binary decision variable $x_{i,t} \in \{0, 1\}$ indicating whether request $i$ begins processing at time $t$. The time horizon is bounded by an upper limit $\overline{T} \le \sum_{i \in [n]} (a_i + o_i)$ to ensure a well-defined formulation. %The scheduling problem can then be rigorously formulated as follows:
The scheduling problem can then be rigorously formulated as:
\begin{subequations} % 使用 subequations 将整个问题组合并编号 (例如 1a, 1b, 1c)
\label{eq:optimization_problem} % 给整个问题一个标签
\begin{alignat}{2}
    % \min \quad & \sum_{i \in [n]} \left( \sum_{t=a_i}^{\overline{T}} t \cdot x_{i,t} + o_i - a_i \right) \label{eq:obj} \\
    \text{OPT} = \min \quad & \sum_{i \in [n]} \left( \sum_{t=a_i}^{\overline{T}} t \cdot x_{i,t} + o_i - a_i \right) \label{eq:obj} \\
    \text{s.t.} \quad & \sum_{t=a_i}^{\overline{T}} x_{i,t} = 1, && \quad \forall i \in [n] \label{eq:c1} \\
    & \sum_{i=1}^{n} \sum_{k=\max\{a_i, t-o_i\}}^{t-1} (s_i + t - k) \cdot x_{i,k} \le M, \quad && \forall t \in [\overline{T}] \label{eq:c2} \\
    & x_{i,t} \in \{0,1\}, \quad && \forall i \in [n], \forall t \in [\overline{T}]\label{eq:c3}
\end{alignat}
\end{subequations}

\textbf{Objective and Constraints Analysis.} Under non-preemptive execution, the objective \eqref{eq:obj} minimizes the total end-to-end latency across all requests, where a request $i$ starting at $t$ incurs a latency of $(t + o_i) - a_i$. We denote the minimum achievable value of this objective as $\text{OPT}$, representing the \textit{hindsight optimal} performance. Constraint \eqref{eq:c1} ensures that every request is scheduled exactly once. Constraint \eqref{eq:c2} restricts the aggregated memory footprint to the GPU capacity $M$ at any time step $t$. Specifically, any active request $i$ that started at $k \in [\max\{a_i, t-o_i\}, t-1]$ imposes a dynamic memory demand of $s_i + t - k$, precisely capturing both its initial prompt allocation and the token-by-token growth. 
Finally, constraint \eqref{eq:c3} defines the binary decision variables.

% \textbf{Key Abstractions.} To bridge mathematical tractability with physical system realities, the formulation relies on three core abstractions. First, it adopts \textit{unified spatio-temporal units}, abstracting away hardware-specific compute disparities by equating each token's generation to a single discrete time step and a standardized memory unit. We empirically assess the fidelity of this abstraction using instrumented vLLM measurements in Appendix~\ref{app:physical_volume_validation}.
% Second, it enforces \textit{non-preemptive execution} to isolate the efficacy of the scheduling metric from hard-to-quantify system-level overheads (e.g., cross-GPU communication or KV cache swapping).
% Finally, evaluating the capacity constraint \eqref{eq:c2} at every step $t$ captures the token-level memory dynamics of continuous batching within the non-preemptive abstraction.
\textbf{Key Abstractions.} To bridge mathematical tractability with physical system realities, the formulation relies on three core abstractions. First, it adopts \textit{unified spatio-temporal units}, abstracting away hardware-specific compute disparities by equating each token's generation to a single discrete time step and a standardized memory unit. We empirically assess the fidelity of this abstraction using instrumented vLLM measurements in Appendix~\ref{app:physical_volume_validation}.
Second, evaluating the capacity constraint \eqref{eq:c2} at every step $t$ captures the token-level continuous batching dynamics. Finally, it enforces \textit{non-preemptive execution} to strictly isolate the fundamental efficacy of the scheduling metric from hard-to-quantify system-level overheads (e.g., cross-GPU communication or KV cache swapping).
% =============================================================================

\subsection{LP Relaxation and Volume-Centric Insight}
\label{subsec:lp_relaxation}
To uncover the structural properties of optimal memory packing, we relax the discrete integer problem into a continuous linear program (LP). We group requests by their total volume $k$, denoting the total count of such requests as $n_k$. Let continuous variables $a_k^t \ge 0$ represent the fractional number of requests from group $k$ completed at time $t$. The relaxed offline optimal, $\text{OPT}_{\text{LP}}$, is formulated as: 

\begin{subequations}
\label{eq:lp_relaxation}
\begin{alignat}{2}
    \text{OPT}_{\text{LP}} := \min \quad & \sum_{t} t \cdot \sum_{k} a_k^t \label{eq:lp_obj} \\[-3pt]
    \text{s.t.} \quad & \sum_{t'=1}^{t} \sum_{k} k \cdot a_k^{t'} \le t \cdot M, \quad && \forall t \label{eq:lp_c1} \\[-3pt]
    & \sum_{t} a_k^t = n_k, \quad && \forall k \label{eq:lp_c2}
\end{alignat}
\end{subequations}
Here, constraint \eqref{eq:lp_c1} bounds the cumulative processed volume by the available memory-time capacity budget $t \cdot M$. 
Analyzing \eqref{eq:lp_obj} from a cost perspective reveals a structural truth: assigning a request completion to time $t$ incurs a direct objective cost of $t$. To minimize the total cost, the optimal policy must aggressively pack the maximum number of requests into the earliest possible time steps. Given the budget $t \cdot M$ in \eqref{eq:lp_c1}, maximizing the item count mathematically dictates prioritizing requests with the smallest volume $k$. 
% This formally proves that sorting by volume is not an ad-hoc heuristic, but the exact optimal structural property strictly dictated by the problem's foundational fractional relaxation.
This formally proves that sorting by volume is not an ad-hoc heuristic, but the exact optimal structural property induced by the fractional relaxation.
% =============================================================================

\subsection{Worst-Case Bounds under Burst Arrivals}
\label{sec:burst}
% 这些bounds的证明可能都得放附录，但正文得放proof Sketch. 以及可能的Remark——解释这个bound在物理系统上意味着什么，比如不需要复杂调参，任意极端情况都接近最优调度

We first analyze the adversarial burst arrival scenario, where a massive influx of requests arrives simultaneously at $t=0$. This models extreme system stress events, such as traffic spikes following product launches, providing a rigorous testbed for worst-case head-of-line blocking resilience. 
% For analytical clarity, we assume a standard safety baseline where the peak memory footprint of any single request does not exceed half of the memory capacity, namely $p_i = s_i + o_i \le M/2$. 
For analytical clarity, we parameterize the workload by its peak per-request memory fraction $\alpha:=\max_i p_i/M$, with $p_i=s_i+o_i$. This is a measured physical quantity rather than a tunable knob: at least $\left\lfloor 1/ \alpha \right\rfloor$ requests can co-reside in memory, so $1/\alpha$ is the system's concurrency headroom. Two anchor points recur below: $\alpha=1/2$, the conservative half-capacity margin under which \cite{jaillet2025online} analyzes its shortest-first policy; and the vanishing-fraction regime $\alpha \to 0$ (equivalently $s_i,o_i = o(M)$), under which \cite{wang2025llm} established the prior best constant of 48.

%To evaluate the performance of our geometry-aware paradigm, 
% We first identify the absolute physical processing limit of the system.
We first establish a new, tighter lower bound on OPT by employing a volume-centric capacity analysis, improving upon the LP-relaxation-based bound from~\cite{jaillet2025online}. The complete proof is in Appendix \ref{sec:appendix_proof_opt_lb}.
\begin{proposition}[Lower Bound of OPT]
\label{prop:opt_lower_bound}
Let $vol_1 \le vol_2 \le \dots \le vol_N$ be the requests sorted by their geometric volumes. The total end-to-end latency of any offline optimal schedule (OPT) satisfies:
\begin{equation}
    \text{TEL}(\text{OPT}) \ge \frac{1}{M} \sum_{j=1}^N \sum_{i=1}^j vol_i
\end{equation}
\end{proposition}

%This bound represents the tightest possible packing of 2D spatio-temporal volumes in a rigid memory-time container $M$. 
%The double summation captures the inevitable backlog inherent to any scheduling sequence. The rigorous derivation is provided in Appendix \ref{sec:appendix_proof_opt_lb}.

Next, we evaluate the efficiency of our SVF. Instead of relying on traditional operations research methodologies---including grouping task lengths into buckets and discretizing execution into localized time windows---our proof directly derives a unified volume-rate certificate from the admission rule.

\begin{theorem}[Worst-case Bound of SVF]
\label{thm:svf-burst}
Under the burst arrival model with $p_i = s_i + o_i \le \alpha M$,
SVF achieves a worst-case approximation ratio
\begin{equation}
  \frac{\text{TEL}(\text{SVF})}{ \text{TEL}(\text{OPT})} \le  1 + \frac{2}{1-\alpha}.
\end{equation}
\end{theorem}

\begin{corollary}[Vanishing-fraction regime]
\label{cor:svf-vanishing}
As $\alpha \to 0$, the bound of Theorem~\ref{thm:svf-burst} approaches $\textbf{3}$. In this regime, the prior best constant is $48$, attained by the \emph{offline} batch-partitioning policy of~\cite{wang2025llm}. SVF improves it to $3$ with a purely \emph{online} greedy admission rule.
\end{corollary}

\begin{corollary}[Conservative half-capacity margin]
\label{cor:svf-half}
At $\alpha = 1/2$---the safety margin adopted in the analysis of~\cite{jaillet2025online}---Theorem~\ref{thm:svf-burst} yields an approximation ratio of at most \(5\).
\end{corollary}

% \begin{proof}[Proof Sketch]
% % The technical obstacle is that a local description of the current batch does not directly yield a latency bound. 
% The key idea is to turn every blocked step into a certificate of volume processing. Fix request $j$, and let $A(t)$ be the active set at a step while $j$ is waiting. If $\sum_{i\in A(t)}p_i \le (1-\alpha) M$, then, since $p_j \le \alpha M$, the system could safely insert $j$; hence $j$'s being blocked certifies $\sum_{i\in A(t)}p_i > (1-\alpha)M$. Summing this certificate over the $W_j$ waiting steps gives $(1-\alpha)M\cdot W_j < \sum_{t<W_j}\sum_{i\in A(t)}p_i.$
% SVF is now essential: while $j$ waits, every active request has smaller volume, so it lies in the SVF prefix $P_j$. After swapping the time and request sums, each predecessor $i\in P_j$ can contribute for at most $o_i$ steps, and $p_i o_i < 2\,\mathrm{vol}_i$. Thus $W_j < \frac{2}{(1-\alpha)M}\sum_{i\in P_j}\mathrm{vol}_i.$
% This is the decisive alignment: the algorithmic waiting time is charged to exactly the same prefix-volume sums that lower-bound OPT in Proposition~\ref{prop:opt_lower_bound}. Summing over all $j$, SVF's queuing latency is at most $\frac{2}{(1-\alpha)}\,\mathrm{TEL}(\mathrm{OPT})$. The remaining $\sum_j o_j$ decoding time is unavoidable for any schedule, giving one more $\mathrm{TEL}(\mathrm{OPT})$. Therefore $\mathrm{TEL}(\mathrm{SVF}) \le (1+\frac{2}{1-\alpha})\,\mathrm{TEL}(\mathrm{OPT})$, yielding the clean $1+\frac{2}{1-\alpha}$ form.
% The complete proof is deferred to Appendix \ref{sec:appendix_proof_svf_cr}.
% \end{proof}

\begin{proof}[Proof Sketch]
The key idea is to turn every blocked step into a certificate of volume processing. Fix request $j$, let $P_j$ be the set of requests preceding $j$ in the SVF ordering, and let $A(t)$ be the active set while $j$ is waiting. Let $h(t)$ be the current head-of-line request. Since SVF follows the fixed priority ordering, $h(t)\in P_j\cup\{j\}$ and $A(t)\subseteq P_j$. If $\sum_{i\in A(t)}p_i \le (1-\alpha)M$, then $p_{h(t)}\le\alpha M$ implies that $h(t)$ could be safely admitted, contradicting that it is blocked. Hence $\sum_{i\in A(t)}p_i>(1-\alpha)M$. Summing over the $W_j$ waiting steps gives $(1-\alpha)MW_j<\sum_{t<W_j}\sum_{i\in A(t)}p_i$.
Each predecessor $i\in P_j$ contributes for at most $o_i$ steps, and $p_io_i<2\,\mathrm{vol}_i$, so $W_j<\frac{2}{(1-\alpha)M}\sum_{i\in P_j}\mathrm{vol}_i$. This aligns the algorithmic waiting cost with the same prefix-volume sums that lower-bound OPT in Proposition~\ref{prop:opt_lower_bound}. Summing over all $j$ gives queuing latency at most $\frac{2}{1-\alpha}\mathrm{TEL}(\mathrm{OPT})$, while $\sum_j o_j\le\mathrm{TEL}(\mathrm{OPT})$. Therefore $\mathrm{TEL}(\mathrm{SVF})\le\left(1+\frac{2}{1-\alpha}\right)\mathrm{TEL}(\mathrm{OPT})$. The complete proof is in Appendix~\ref{sec:appendix_proof_svf_cr}.
\end{proof}

% \textbf{Remark.} The theorem establishes a worst-case guarantee: regardless of how extreme or adversarial the traffic bursts and task length distributions may be, our algorithm's total latency will never exceed five times that of the omniscient offline optimal.
% Achieving $CR\le 5$ requires a fundamental departure from existing proof architectures. Classical analyses rely on local combinatorial reductions, such as time-window discretizations and length bucketing~\cite{jaillet2025online}. Alternatively, the recent approach with $CR\le 48$ attempts to relax the overall policy into strictly constrained proxy policies~\cite{wang2025llm}. Both paradigms inherently accumulate massive analytical slack. Instead of following these paths, our proof introduces a novel macroscopic certificate: every blocked step directly implies an active peak-memory mass strictly greater than $M/2$. Upon spatio-temporal integration, this translates into a prefix-volume charge. Because SVF and the OPT lower bound share the identical volume-centric ordering, the algorithmic cost perfectly mirrors the theoretical baseline. This exact structural alignment completely bypasses intermediary relaxations and discretization losses, directly yielding the clean $4+1$ form.
\begin{remark}
The theorem certifies a worst-case guarantee: no matter how adversarial the traffic burst or task-length distribution, SVF's total latency never exceeds $1+\frac{2}{1-\alpha}$ times that of the omniscient offline optimum. Here $\alpha$ is a physical property of the workload---the largest single-request memory fraction, so $1/\alpha$ is the system's concurrency headroom---rather than a tunable constant. 
%The limit $\alpha\to1$, where the bound diverges, is the degenerate regime in which memory holds essentially one request at a time and continuous batching, hence scheduling itself, ceases to exist; the divergence faithfully reflects the vanishing of all scheduling freedom rather than any weakness of SVF. 
Production serving instead operates at small $\alpha$, where the bound stays close to $3$ (Table~\ref{tab:alpha-cr}).

Attaining this bound requires a departure from existing proof architectures. Classical analyses rely on local combinatorial reductions such as time-window discretization and length bucketing~\cite{jaillet2025online}, while the prior $48$-bound relaxes the policy into an offline, batch-partitioning procedure~\cite{wang2025llm}. Both inherently accumulate analytical slack. Our proof instead introduces a macroscopic certificate: every blocked step implies an active peak-memory mass exceeding $(1-\alpha)M$, which upon spatio-temporal integration charges precisely the prefix-volume sums that lower-bound OPT. SVF and the OPT lower bound share the identical volume-centric ordering, bypassing intermediary relaxations and directly yielding the clean $1+\frac{2}{1-\alpha}$ form.
\end{remark}

\begin{table}[t]
  \centering
  % \caption{$1+\frac{2}{1-\alpha}$ as a function of the peak memory fraction
  % $\alpha = \max_i p_i / M$. Production serving operates in the
  % high-concurrency regime, where the bound stays near $3$.}
  \caption{$1+\frac{2}{1-\alpha}$ as a function of the peak memory fraction $\alpha = \max_i p_i / M$. Production serving operates in the high-concurrency regime, where the bound stays near $3$.}
  \label{tab:alpha-cr}
  \begin{tabular}{c*{6}{c}}
    \toprule
    $\alpha = \max_i p_i / M$
      & $1/2$ & $1/4$ & $1/10$ & $1/20$ & $1/100$ & $\to 0$ \\
    \midrule
    Min.\ concurrency
      & $2$ & $4$ & $10$ & $20$ & $100$ & $\to \infty$ \\
    $\mathrm{AR} \le 1 + 2/(1-\alpha)$
      & $5.00$
      & $\mathbf{3.67}$
      & $\mathbf{3.22}$
      & $\mathbf{3.11}$
      & $\mathbf{3.02}$
      & $\to \mathbf{3}$ \\
    \bottomrule
  \end{tabular}
\end{table}

We further demonstrate the resilience of our geometric paradigm by establishing a deterministic bound with merely a single bit of information. 1-bit SVF bounds the approximation ratio by $\mathcal{O}(T)$, where $T$ elegantly corresponds to the bounded system hyperparameter of maximum model length.

\begin{theorem}[Worst-case Bound of 1-bit SVF]
\label{thm:svf_1bit_burst}
Under the burst arrival model with $s+o\le T$, 1-bit SVF algorithm configures a binary classification threshold $\theta = \sqrt{T}$ and assigns proxy lengths $O_0 = T^{1/4}$ for short requests ($o \le \theta$) and $O_1 = T^{3/4}$ for long requests ($o > \theta$), achieving an approximation ratio strictly bounded by $\frac{ \text{TEL}(\text{1-bit SVF})}{ \text{TEL}(\text{OPT})}= \mathcal{O}(T)$.
\end{theorem}

% \begin{remark}
% Since $T$ is a static model configuration, this formally guarantees that no matter how massive the burst workload becomes ($N \to \infty$), the worst-case degradation is permanently capped by a system constant. Detailed algebraic proofs are provided in Appendix \ref{sec:appendix_proof_1bit}.
% \end{remark}
\begin{remark}
The $\mathcal{O}(T)$ bound depends on the per-request length range but is independent
of the burst size $N$. Thus, for a fixed model configuration, the worst-case factor
does not deteriorate as the number of simultaneously arriving requests grows.
Detailed proofs are provided in Appendix~\ref{sec:appendix_proof_1bit}.
\end{remark}

% =============================================================================
\subsection{Stochastic Analysis under Poisson Arrivals}
To evaluate the steady-state performance of our geometry-aware algorithms, we analyze the system under a Poisson arrival process with rate $\lambda$. \cite{xiang2025servegen} demonstrates that the output generation lengths follow an exponential distribution. To capture the discrete, token-by-token mechanics, we adopt the discrete analog of the exponential distribution and model the discrete output length as a Geometric distribution $o_i \sim \text{Geo}(\mu)$, yielding an expected length $\mathbb{E}[o] = 1/\mu$. In physical systems, output lengths are bounded by the peak-memory envelope $s+o\leq \alpha M$. Hence, the untruncated geometric law is interpreted as an approximation when its infeasible tail is negligible, as detailed in Appendix~\ref{sec:appendix_proof_poisson}.

% We define the worst-case system utilization as $\rho = \frac{\lambda E[vol]}{(1-\alpha)M/2}$. A stable system requires $\rho \in [0, 1)$.
We define the conservative load parameter $\rho=\lambda\mathbb{E}[\mathrm{vol}]/((1-\alpha)M/2)\in [0, 1)$.

\begin{theorem}[Stochastic Bound of SVF]
\label{thm:svf_poisson}
% Under the Poisson arrival process with rate $\lambda$ and the peak memory constraint $s+o \le \alpha M$, as a physical system boundary, 
Under the stochastic-analysis setting described above, with Poisson arrivals of rate $\lambda$ and the peak-memory constraint $s+o \le \alpha M$,
% SVF is bounded by:
the expected-latency ratio for SVF is bounded by:
\begin{equation}
    \frac{ \mathbb E[TEL({\text{SVF}})] }{ \mathbb E[TEL({\mathrm{OPT}})] } \le 1 + \frac{2}{(1-\alpha)(1-\mu)(1-\rho)}.
\end{equation}
\end{theorem}

\begin{proof}[Proof Sketch]
The proof relies on macroscopic volume conservation. By plotting the cumulative arrival volume against the cumulative processed volume, the area between these curves represents the total volumetric backlog. We establish that SVF maintains an expected processing rate $\bar{r} > (1-\alpha)M/2$ whenever the queue is non-empty. Leveraging the memoryless property of the Geometric distribution, we quantify the expected non-preemptive residual volume $\mathbb{E}[U_{run}]$. Applying the Harris inequality separates the correlated volume and waiting time, allowing us to align the upper bound of SVF's queuing delay with the area-based lower bound of OPT.
\end{proof}

\begin{theorem}[Stochastic Bound of 1-bit SVF]
\label{thm:svf_1bit_poisson}
Under the same stochastic-analysis setting described above, 
% for the 1-bit SVF algorithm configured with a threshold $\theta$ and proxy lengths assigned as their conditional expectations—
consider the 1-bit SVF algorithm with threshold $\theta$ and proxy lengths given by
their conditional expectations under the geometric reference model,
specifically, $O_0 = \frac{1}{\mu} - \frac{\theta(1-\mu)^\theta}{1 - (1-\mu)^\theta}$ for short requests ($o \le \theta$) and $O_1 = \theta + \frac{1}{\mu}$ for long requests ($o > \theta$). The expected-latency ratio is bounded by:
\begin{equation}
    \frac{ \mathbb E[TEL({1\text{-bit SVF}})] }{ \mathbb E[TEL({\mathrm{OPT}})] } \le 1 + \left( \frac{\mathbb{E}[vol]}{\mathbb{E}[vol] - \epsilon} \right) \frac{2}{(1-\alpha)(1-\mu)(1-\rho)}
\end{equation}
where the volume distortion penalty is $\epsilon = \frac{1}{2} \left( \frac{\theta (1-\mu)^\theta}{1 - (1-\mu)^\theta} \right)^2$.
\end{theorem}

\begin{remark}
Theorems~\ref{thm:svf_poisson} and \ref{thm:svf_1bit_poisson} establish steady-state expected-latency bounds under Poisson arrival processes. The performance gap of 1-bit SVF relative to full-information SVF is characterized by the volume-distortion parameter $\epsilon$. Crucially, as the threshold $\theta$ increases, $\epsilon$ decays exponentially, rapidly narrowing the gap between the two bounds. Consequently, 1-bit SVF preserves the same qualitative queuing stability guarantees as SVF. Complete proofs are deferred to Appendix~\ref{sec:appendix_proof_poisson}.
\end{remark}

\section{Experiments}
\label{sec:experiments}
\begin{table*}[t]
\centering
\caption{Performance comparison under burst arrivals across different models and benchmarks. Avg/P95 Latency in s/tok; Throughput in tok/s. Best non-oracle result per column is \textbf{bold}.}
\label{tab:burst_main_expanded}
\small
\resizebox{\textwidth}{!}{%
\begin{tabular}{l rrr|rrr | rrr|rrr}
\toprule
\multirow{3}{*}{\textbf{Algorithm}} & \multicolumn{6}{c|}{\textbf{Llama-3.1-8B-Instruct}} & \multicolumn{6}{c}{\textbf{Llama-3.1-70B-Instruct}} \\
\cmidrule(lr){2-7} \cmidrule(lr){8-13}
& \multicolumn{3}{c|}{\textbf{LMSYS}} & \multicolumn{3}{c|}{\textbf{LongBench}} & \multicolumn{3}{c|}{\textbf{LMSYS}} & \multicolumn{3}{c}{\textbf{LongBench}} \\
\cmidrule(lr){2-4} \cmidrule(lr){5-7} \cmidrule(lr){8-10} \cmidrule(lr){11-13}
& Avg Lat. & P95 Lat. & Thpt. & Avg Lat. & P95 Lat. & Thpt. & Avg Lat. & P95 Lat. & Thpt. & Avg Lat. & P95 Lat. & Thpt. \\
\midrule
FCFS & 2.4592 & 10.2364 & 3338.1 & 125.8883 & 393.1413 & 23.9 & 3.0666 & 14.0373 & 2467.9 & 226.2472 & 721.3823 & 10.9 \\
SJF & 1.0087 & 4.8632 & 3229.9 & 106.1588 & 300.2542 & 25.9 & 1.8275 & 11.4283 & 2444.3 & 173.9461 & 442.1720 & 11.9 \\
1-bit SVF & 1.5614 & 6.2247 & \textbf{3361.8} & 72.5786 & 249.2350 & \textbf{28.7} & 1.9850 & \textbf{6.8657} & \textbf{2492.1} & 141.4835 & 524.7749 & 12.4 \\
SVF & \textbf{0.8901} & \textbf{3.7570} & 3239.4 & \textbf{70.0169} & \textbf{215.1756} & 28.2 & \textbf{1.6090} & 9.0097 & 2360.6 & \textbf{140.4193} & \textbf{420.2404} & \textbf{12.5} \\
\midrule
Oracle SJF & 0.1181 & 0.1497 & -- & 74.3010 & 134.4113 & -- & 0.2045 & 0.2886 & -- & 144.3262 & 242.7121 & -- \\
Oracle SVF & 0.1103 & 0.1950 & -- & 52.5263 & 128.7662 & -- & 0.1792 & 0.3348 & -- & 108.4169 & 251.4195 & -- \\
\bottomrule
\end{tabular}%
}
\end{table*}

\begin{figure}[t] % TODO：感觉后续两个字图又可以合并，因为可以共享图例把图例变大点
    \centering
    \begin{subfigure}[b]{0.49\linewidth}
        \centering
        \includegraphics[width=\linewidth]{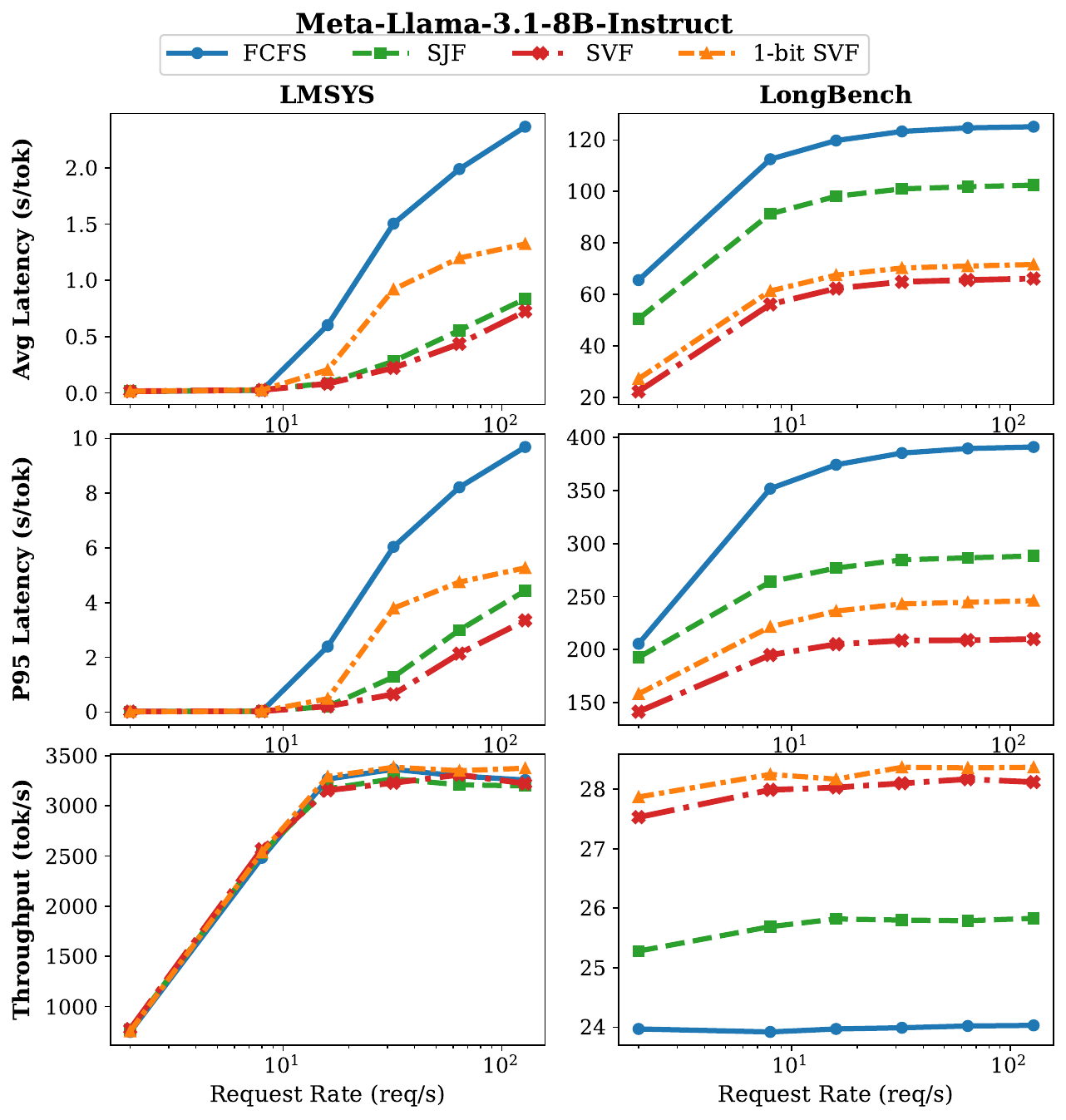}
        \caption{Llama-3.1-8B-Instruct}
        \label{fig:subfig_a} % 用于引用第一个子图的标签
    \end{subfigure}
    \begin{subfigure}[b]{0.49\linewidth}
        \centering
        \includegraphics[width=\linewidth]{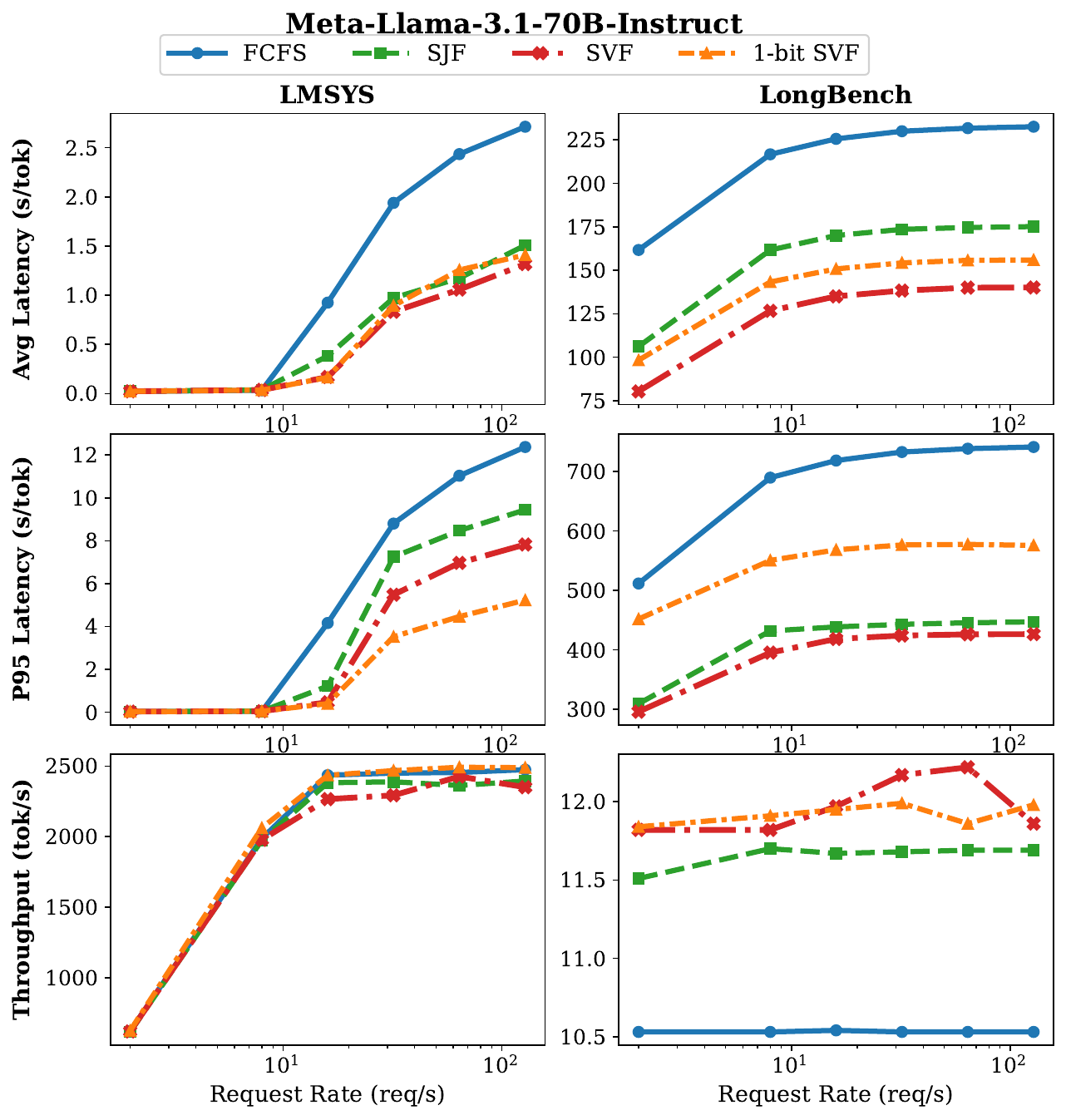}
        \caption{Llama-3.1-70B-Instruct}
        \label{fig:subfig_b} % 用于引用第二个子图的标签
    \end{subfigure}
    % --- 主标题 ---
    \caption{Practical Performance under Poisson Arrivals}
    \label{fig:practical_performance}
\end{figure}

\subsection{Evaluation Setup}
\label{EvaluationSetup}
\textbf{Testbed.} 
All experiments are conducted on a single high-performance compute node equipped with 8 NVIDIA A100 (80 GB) GPUs. The underlying continuous batching engine is built upon vLLM.

\textbf{Serving Models.} 
We evaluate our scheduling algorithms using two open-source models: Meta-Llama-3.1-8B-Instruct and Meta-Llama-3.1-70B-Instruct~\cite{grattafiori2024llama}. The 8B model is deployed on a single GPU, whereas the 70B model is served across all 8 GPUs utilizing tensor parallelism~\cite{shoeybi2019megatron}. Both models operate in FP16/BF16 precision, with the context window configured to 65,536 tokens and the maximum generation length capped at 4,096 tokens per request.

\textbf{Workloads.} %这里得简单 但又不深入的提起选取benchmark的motivation
We simulate real-world heterogeneous traffic using two distinct datasets: \textit{LMSYS-Chat}~\cite{zheng2023lmsyschat1m}, which represents the high-concurrency, interactive traffic common in chat applications, and \textit{LongBench}~\cite{bai2023longbench}, which introduces the memory-intensive jobs that often challenge serving systems. For each dataset, we randomly sample 2,000 requests exclusively for serving evaluation, reserving the non-overlapping remainder strictly for training the length predictors.% The random seed is fixed at 42 to ensure exact reproducibility. 

\textbf{Arrival Patterns.} 
To comprehensively assess scheduling robustness, we evaluate two distinct arrival models. We first employ an \textbf{Adversarial Burst} scenario where all requests arrive simultaneously at $t=0$, designed to probe the system's peak scheduling capacity and head-of-line blocking resilience. We also simulate a \textbf{Stochastic Poisson} process with query-per-second rates QPS $\in \{2, 8, 16, 32, 64, 128\}$ to analyze steady-state performance under various levels of system load.

\textbf{Predictor Setup.} %Predictor Training 作为实验setup的一个部分，就讲述借鉴了什么训练出来的predictor，由于我们的卖点是算法而不是predictor，我们并没有对此进行精细的调优。具体predictor的性能见附录。
Since our core contribution lies in the novel scheduling paradigm, we employ an \textit{off-the-shelf} architecture for our length predictors from \cite{qiu2024efficient} to ensure a fair evaluation focused squarely on our algorithmic innovations.
For regression prediction, we employ a \texttt{BERT-base} backbone~\cite{devlin2019bert}. In contrast, for our lightweight \textit{1-bit} variant, we deploy a much smaller \texttt{BERT-tiny}~\cite{DBLP:journals/corr/abs-1908-08962} classifier.%, thereby drastically minimizing the prediction inference overhead. 

% \textbf{Scheduling Baselines.} % TODO：有点顾虑baselines这一块会不会被说，需不需要防御性表达
% The evaluated baselines include:
% \begin{itemize}
%     \item \textbf{FCFS:} The native vLLM baseline, prioritizing requests strictly by their arrival time.
%     \item \textbf{SJF:} The traditional 1D heuristic, prioritizing requests by the exact output length predicted.
%     \item \textbf{Oracle-SJF \& Oracle-SVF:} Idealized schedulers providing a idealized upper bound. They leverage ground-truth output lengths for both optimal prioritization~\cite{tao2025prompt} and for setting each request's maximum generation length. This ensures a perfect match between pre-allocated resources and actual consumption, serving as an idealized performance ceiling.

% \end{itemize}
\textbf{Scheduling Baselines.} 
To strictly isolate our algorithmic contributions from composite system engineering, we evaluate our proposed paradigm against fundamental scheduling primitives. %While recent end-to-end systems introduce complex execution mechanisms (e.g., proactive swapping), 
For recent complex execution mechanisms (e.g., proactive swapping), our geometry-aware approach serves as a \textit{composable scheduling policy} that can be superimposed atop these optimizations, potentially yielding additional gains. Thus, we select \textbf{FCFS} as the native baseline and \textbf{SJF} as the classical 1D heuristic. 
To establish the absolute theoretical limits, we introduce \textbf{Oracle-SJF} and \textbf{Oracle-SVF} as idealized schedulers, leveraging ground-truth output lengths for both optimal prioritization~\cite{tao2025prompt} and for setting each request's maximum generation length. This ensures a match between pre-allocated resources and actual consumption, serving as an idealized performance ceiling.
%regardless of the underlying execution mechanisms they are coupled with.

\textbf{Evaluation Metrics.} % TODO：等实验结果出来后，这部分还要优化下表达
To eliminate the inherent bias introduced by widely varying generation lengths across the datasets, we focus on \textit{Per-token Latency}, defined as the end-to-end latency of a request divided by its actual output token count. We report both the Mean and the 95th percentile (P95) tail per-token latency. Additionally, we track \textit{Token Throughput}, calculated as the total number of generated tokens divided by the total experimental duration.

% \begin{figure}
%     \centering
%     \includegraphics[width=0.80\linewidth]{Sources/fig1_poisson_main.pdf}
%     \caption{Practical Performance under Poisson Arrivals [waiting for the results of 70B]}
%     \label{fig:practical_perf}
% \end{figure}

\begin{figure}[t] % TODO：感觉后续两个字图又可以合并，因为可以共享图例把图例变大点
    \centering
    \begin{subfigure}[b]{0.49\linewidth}
        \centering
        \includegraphics[width=\linewidth]{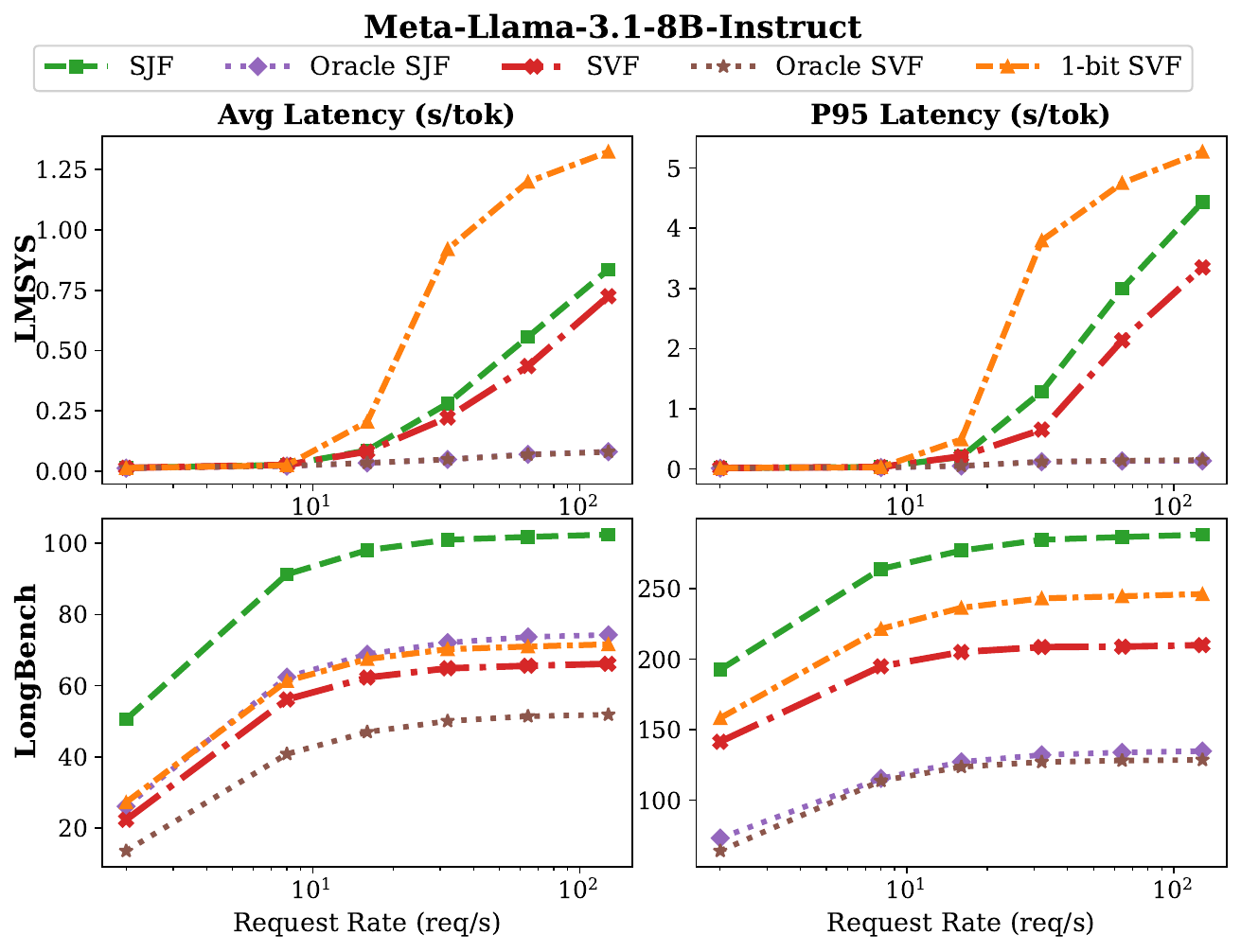}
        \caption{Llama-3.1-8B-Instruct}
        \label{fig:subfig_a} % 用于引用第一个子图的标签
    \end{subfigure}
    \begin{subfigure}[b]{0.49\linewidth}
        \centering
        \includegraphics[width=\linewidth]{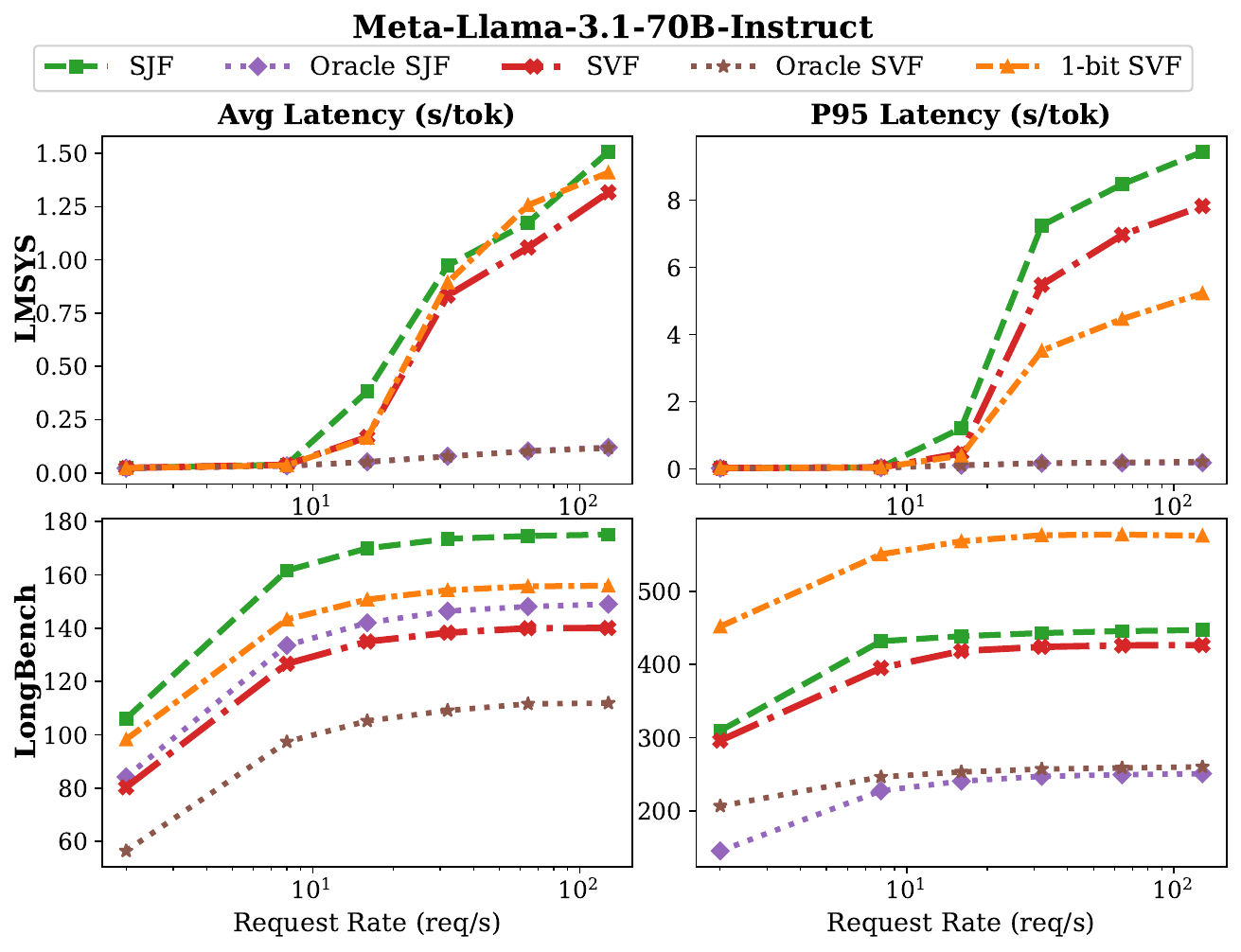}
        \caption{Llama-3.1-70B-Instruct}
        \label{fig:subfig_b} % 用于引用第二个子图的标签
    \end{subfigure}
    % --- 主标题 ---
    \caption{Validation of Theoretical Claims under Poisson Arrivals}
    \label{fig:main_figure} % 用于引用整个图组的标签
\end{figure}

\subsection{End-to-End Scheduling Performance}
\label{subsec:end_to_end}

% We first evaluate the end-to-end performance of our geometry-aware paradigm across both burst and Poisson arrival patterns. To dissect the performance gains, we structure our analysis into two phases: assessing practical deployment viability and evaluating proximity to theoretical bounds.
We first evaluate the end-to-end performance of our geometry-aware paradigm across both burst and Poisson arrival patterns. To dissect the performance gains, we structure our analysis into two phases: assessing practical deployment viability and validating theoretical soundness. Throughput is omitted for the Oracle baselines as they set the maximum generation tokens to the ground-truth length, rendering any online throughput calculation fundamentally artificial.

\textbf{Practical Performance.} 
%To demonstrate the immediate practical benefits of our algorithms, we compare the deployment-ready policies: \{FCFS, SJF, SVF, 1-bit SVF\}. As shown in Table \ref{tab:burst_main_expanded}, under the adversarial burst scenario, [result analysis] 
% 1. 跨到达场景模型，SVF始终保持最低延迟。 并且保持最低尾部延迟，除了70Blmsys下略高于1bit SVF。
% 2. 跨到达场景，SVF的吞吐量始终高于SJF，除了70Blmsys。
% 3. 跨到达场景，1bit SVF相比于FCFS始终实现更低平均、更低尾部和更高吞吐。
% 4. 跨到达场景，1bit SVF的吞吐稳定高于SJF。1bit由于信息精度的不足，在lmsys这种chatbot场景下平均延迟略高于SJF，但在longbench中，由于SJF的一维局限性超越了信息精度的优势，1bit实现了在平均延迟对于SJF的超越。
To demonstrate the practical benefits of our algorithms, we compare the deployment-ready policies: \{FCFS, SJF, SVF, 1-bit SVF\} and the empirical results are shown in Table \ref{tab:burst_main_expanded} and Figure \ref{fig:practical_performance}. % reveal a profound structural consistency in their performance dynamics regardless of the underlying workload pressure.
First, across all arrival patterns, benchmarks, and model scales, SVF consistently achieves the \textbf{lowest average latency}. Furthermore, it largely secures the lowest P95 tail latency and maintains a higher system throughput than SJF, with the sole exception being the 70B model under the LMSYS workload where 1-bit SVF marginally leads in tail latency and throughput. Second, 1-bit SVF outperforms FCFS, delivering lower average latency, lower tail latency, and higher throughput across all scenarios. When evaluated against SJF, 1-bit SVF consistently achieves higher throughput. In terms of latency, the inherent precision loss of binary classification leads to slightly higher average latency in LMSYS compared to SJF. However, in memory-intensive scenarios (LongBench), the fundamental dimensional flaw of SJF's 1D time-centric metric overshadows the precision disadvantage, allowing 1-bit SVF to surpass SJF in average latency. 

To further assess robustness to workload heterogeneity, we evaluate a mixed workload that interleaves 80\% LMSYS and 20\% LongBench requests in Appendix~\ref{app:mixed_workload}, where the advantage of geometry-aware scheduling remains consistent.

% Appendices~\ref{app:mixed_workload} and~\ref{sec:real_trace} further validate robustness under mixed workloads and real-world arrivals, respectively, where geometry-aware scheduling consistently improves latency and throughput.

\textbf{Remark on Tail Latency and Scheduling Fairness.}
In classical 1D scheduling, greedy heuristics like SJF inherently compromise fairness by starving long requests.
%Counter-intuitively, our geometry-aware greedy strategy achieves the opposite. 
However, by greedily executing requests with the smallest 2D volume, SVF rapidly releases bounded KV-cache memory back to the system. This accelerated memory recuperation effectively breaks global head-of-line blocking, allowing larger requests to be admitted significantly earlier. %The empirical P95 latencies perfectly confirm that our volume-centric greediness practically guarantees system-wide fairness without indefinite starvation.
The empirical P95 latencies suggest that the gains in average latency do not come at the expense of substantially worse tail behavior.

% \textbf{Validation of Theoretical Claims.} 
%To strictly validate our theoretical claims regarding the fundamental flaws of SJF, we conduct an ablation on \{SJF, SVF, Oracle-SJF, Oracle-SVF\}. As depicted in Figure \ref{fig:main_figure} and Table \ref{tab:burst_main_expanded}, [result analysis]
% 先做中文的结果分析：切入维度为
% 1. oracle相互对比的话，oracle SVF的平均延迟稳定为最低。
% 2. 值得注意的是，burst场景（head of blocking testbed）和poisson场景下longbench场景中，SJF的一维局限性被充分暴露出来，其oracle的平均延迟低于SVF和1bit SVF，除了70B longbench中1bit SVF高于oracle SJF。
\textbf{Ablation Study.} 
% We conduct an ablation on \{SJF, SVF, Oracle-SJF, Oracle-SVF\}. 
As depicted in Figure \ref{fig:main_figure} and Table \ref{tab:burst_main_expanded}, Oracle-SVF establishes the performance floor, consistently achieving the lowest average latency.
%Most crucially, under both the burst and the Poisson scenario in LongBench setting, the fatal dimensional flaw of SJF is fully exposed: the theoretical Oracle-SJF performs worse in average latency than our practical, non-omniscient 1-bit SVF and SVF.
Most crucially, in memory-intensive workloads (LongBench) under both arrival patterns, the fatal dimensional flaw of SJF is fully exposed: Oracle-SJF with perfect future knowledge of output lengths yields worse average latency than our non-omniscient SVF. This empirically proves that time-centric metrics are inherently misaligned with LLM dynamics, and 2D geometry is a mandatory paradigm shift for optimal scheduling.

\subsection{Real-World Trace Evaluation}
\label{sec:real_trace}

\textbf{Experimental Design.}
To evaluate our scheduling policies under realistic time-varying traffic, we conduct a trace-driven experiment using BurstGPT~\cite{wang2025burstgpt}, which provides real request arrival timestamps and input/output token-length profiles from LLM serving workloads. Since the trace does not contain executable prompts, we instantiate a contiguous 2,000-request window by matching each trace event, without replacement, to an LMSYS or LongBench request with a corresponding input--output length profile. We uniformly rescale the trace to a mean offered load of 128 requests/s, while retaining substantial temporal variation in both arrival intensity and request composition. Predictor training uses a separate pool of 2,000 requests.

\begin{table}[t]
\centering
\small
\setlength{\tabcolsep}{3.8pt}
\caption{Performance under a non-stationary real-world arrival trace.
Best non-oracle result per metric is \textbf{bold}.}
\label{tab:real_trace}
\begin{tabular}{lrrrrrr}
\toprule
Metric & FCFS & SJF & 1-bit SVF & SVF & Oracle SJF & Oracle SVF \\
\midrule
Avg Lat. (s/tok) & 24.5554 & 25.5967 & \textbf{9.3201} & 10.0674 & 11.7257 & 4.3134 \\
P95 Lat. (s/tok) & 97.1346 & 100.6038 & \textbf{58.1221} & 60.5942 & 31.7311 & 25.4869 \\
Thpt. (tok/s) & 660.4 & 650.3 & \textbf{684.0} & 673.3 & -- & -- \\
\bottomrule
\end{tabular}
\end{table}

\textbf{Results.}
As shown in Table~\ref{tab:real_trace}, geometry-aware scheduling maintains a clear and consistent advantage under the non-stationary real-world trace, with both SVF variants substantially reducing average and tail latency while improving throughput over FCFS and SJF. Notably, 1-bit SVF achieves the best non-oracle performance across all three metrics, %reducing average and P95 per-token latency over SJF by $63.6\%$ and $42.2\%$, respectively, while increasing throughput by $5.2\%$. This result 
which highlights an appealing property of the geometry-aware design: a coarse geometric abstraction can preserve the dominant scheduling signal even under strongly time-varying traffic, enabling robust practical decisions with only lightweight prediction. Together with the strong performance of Oracle SVF, the results reinforce memory--time geometry as the key scheduling signal.

\subsection{System Overhead Profiling}
\label{subsec:overhead}

\begin{table}[t]
\centering
\small
\setlength{\tabcolsep}{4pt}
\caption{Computational overhead of the complete policies. $\text{Overhead} = \text{Policy Time} / \text{E2E Time}$.}
\label{tab:overhead}
\begin{tabular}{lllrrrr}
\toprule
Policy
& Predictor
& Benchmark
& E2E Time (s)
& Predictor Time (s)
& Policy Time (s)
& Overhead (\%) \\
\midrule
\multirow{2}{*}{SVF}
& \multirow{2}{*}{Regression}
& LMSYS     & 5687.94  & 1.2013 & 2.1990 & 0.0387 \\
&
& LongBench & 16312.59 & 1.6552 & 6.3982 & 0.0392 \\
\multirow{2}{*}{1-bit SVF}
& \multirow{2}{*}{Classifier}
& LMSYS     & 5811.51  & 0.1384 & 1.1799 & 0.0203 \\
&
& LongBench & 14260.91 & 0.6426 & 5.3322 & 0.0374 \\
\bottomrule
\end{tabular}
\end{table}

\textbf{Experimental Design.} 
In LLM serving systems, auxiliary modules must execute with low latency to ensure that GPU compute resources are dedicated to model inference. To demonstrate the lightweight nature of our approach, we profile the complete scheduling-policy path %, defining $T_{\mathrm{policy}}=T_{\mathrm{predictor}}+T_{\mathrm{priority}}+T_{\mathrm{admission}}+T_{\mathrm{scheduler}}$, 
on Meta-Llama-3.1-8B-Instruct by submitting $N{=}200$ concurrent requests. This simultaneous submission~\cite{fu2024efficient} forces the schedulers to process the pending queue in dense batches, faithfully capturing the amortized computational cost under realistic, high-concurrency production workloads.

\textbf{Results.}
As shown in Table \ref{tab:overhead}, the complete scheduling-policy overhead remains \textbf{below $\textbf{0.04\%}$} in all evaluated settings.
%This empirically confirms that our prediction-based policies function as strictly lightweight, plug-and-play modules. The infinitesimal overhead guarantees that the substantial latency reductions and throughput improvements yielded by our geometry-aware sorting translate entirely into system performance gains.
This empirically confirms our algorithms as lightweight, plug-and-play modules, ensuring that the substantial gains from our policies translate directly into system-wide improvements.
\section{Conclusion}

In this work, we have established a convergent theoretical and empirical framework for LLM inference scheduling under memory constraints. 
%Our SVF algorithm and its minimalist 1-bit variant successfully bridge the gap between rigorous mathematical stability and system-level efficiency. 
Our SVF algorithm and its minimalist 1-bit variant embody our geometry-aware approach, successfully translating insights from the rigorous theoretical bounds into tangible gains in system practice. Beyond its technical merits, our geometry-aware approach offers significant implications for sustainable and equitable AI deployment. We provide a dedicated discussion on these broader societal impacts in Appendix \ref{checklist10}.
%In this work, we identify the fundamental structural misalignment between legacy 1D time-centric scheduling and the 2D spatio-temporal growth of the KV cache. Based on a comprehensive 2D geometric modeling, we formally prove that our volume-centric approach achieves a constant worst-case competitive ratio ($CR \le 5$) under adversarial traffic scales. Practical evaluations on the vLLM engine demonstrate that our paradigm yields substantial net gains in average latency, tail latency. 
Looking forward, while our current implementation focuses on single-worker settings, extending this geometry-aware paradigm to multi-GPU clusters presents a promising future direction, necessitating new load-balancing strategies that account for inter-worker communication and evolving 2D memory footprints across nodes.

\section*{Acknowledgment}
This work was supported by National Natural Science Foundation of China (No.62472428), Public Computing Cloud, Renmin University of China, the fund for building world-class universities (disciplines) of Renmin University of China.
%%%%%%%%%%%%%%%%%%%%%%%%%%%%%%%%%%%%%%%%%%%%%%%%%%%%%%%%%%%%
% \newpage
% \section*{References}

\small
\bibliographystyle{unsrt}
\bibliography{refs}

@inproceedings{yu2022orca,
  title={Orca: A distributed serving system for $\{$Transformer-Based$\}$ generative models},
  author={Yu, Gyeong-In and Jeong, Joo Seong and Kim, Geon-Woo and Kim, Soojeong and Chun, Byung-Gon},
  booktitle={16th USENIX symposium on operating systems design and implementation (OSDI 22)},
  pages={521--538},
  year={2022}
}

@inproceedings{kwon2023efficient,
  title={Efficient memory management for large language model serving with pagedattention},
  author={Kwon, Woosuk and Li, Zhuohan and Zhuang, Siyuan and Sheng, Ying and Zheng, Lianmin and Yu, Cody Hao and Gonzalez, Joseph and Zhang, Hao and Stoica, Ion},
  booktitle={Proceedings of the 29th symposium on operating systems principles},
  pages={611--626},
  year={2023}
}

@inproceedings{kaffes2019shinjuku,
  title={Shinjuku: Preemptive Scheduling for $\{$$\mu$second-scale$\}$ Tail Latency},
  author={Kaffes, Kostis and Chong, Timothy and Humphries, Jack Tigar and Belay, Adam and Mazi{\`e}res, David and Kozyrakis, Christos},
  booktitle={16th USENIX Symposium on Networked Systems Design and Implementation (NSDI 19)},
  pages={345--360},
  year={2019}
}

@article{wu2023fast,
  title={Fast distributed inference serving for large language models},
  author={Wu, Bingyang and Zhong, Yinmin and Zhang, Zili and Liu, Shengyu and Liu, Fangyue and Sun, Yuanhang and Huang, Gang and Liu, Xuanzhe and Jin, Xin},
  journal={arXiv preprint arXiv:2305.05920},
  year={2023}
}

@article{jin2023s,
  title={$ S^{3}$: Increasing GPU Utilization during Generative Inference for Higher Throughput},
  author={Jin, Yunho and Wu, Chun-Feng and Brooks, David and Wei, Gu-Yeon},
  journal={Advances in Neural Information Processing Systems},
  volume={36},
  pages={18015--18027},
  year={2023}
}

@article{fu2024efficient,
  title={Efficient llm scheduling by learning to rank},
  author={Fu, Yichao and Zhu, Siqi and Su, Runlong and Qiao, Aurick and Stoica, Ion and Zhang, Hao},
  journal={Advances in Neural Information Processing Systems},
  volume={37},
  pages={59006--59029},
  year={2024}
}

@article{zheng2023response,
  title={Response length perception and sequence scheduling: An llm-empowered llm inference pipeline},
  author={Zheng, Zangwei and Ren, Xiaozhe and Xue, Fuzhao and Luo, Yang and Jiang, Xin and You, Yang},
  journal={Advances in Neural Information Processing Systems},
  volume={36},
  pages={65517--65530},
  year={2023}
}

@inproceedings{stojkovic2025dynamollm,
  title={Dynamollm: Designing llm inference clusters for performance and energy efficiency},
  author={Stojkovic, Jovan and Zhang, Chaojie and Goiri, {\'I}{\~n}igo and Torrellas, Josep and Choukse, Esha},
  booktitle={2025 IEEE International Symposium on High Performance Computer Architecture (HPCA)},
  pages={1348--1362},
  year={2025},
  organization={IEEE}
}

@inproceedings{cheng2024enabling,
  title={Enabling efficient batch serving for lmaas via generation length prediction},
  author={Cheng, Ke and Hu, Wen and Wang, Zhi and Du, Peng and Li, Jianguo and Zhang, Sheng},
  booktitle={2024 IEEE International Conference on Web Services (ICWS)},
  pages={853--864},
  year={2024},
  organization={IEEE}
}

@article{qiu2024efficient,
  title={Efficient interactive llm serving with proxy model-based sequence length prediction},
  author={Qiu, Haoran and Mao, Weichao and Patke, Archit and Cui, Shengkun and Jha, Saurabh and Wang, Chen and Franke, Hubertus and Kalbarczyk, Zbigniew T and Ba{\c{s}}ar, Tamer and Iyer, Ravishankar K},
  journal={arXiv preprint arXiv:2404.08509},
  year={2024}
}

@inproceedings{qiu2024power,
  title={Power-aware deep learning model serving with $\{$$\mu$-Serve$\}$},
  author={Qiu, Haoran and Mao, Weichao and Patke, Archit and Cui, Shengkun and Jha, Saurabh and Wang, Chen and Franke, Hubertus and Kalbarczyk, Zbigniew and Ba{\c{s}}ar, Tamer and Iyer, Ravishankar K},
  booktitle={2024 USENIX Annual Technical Conference (USENIX ATC 24)},
  pages={75--93},
  year={2024}
}

@article{shahout2024don,
  title={Don't Stop Me Now: Embedding Based Scheduling for LLMs},
  author={Shahout, Rana and Malach, Eran and Liu, Chunwei and Jiang, Weifan and Yu, Minlan and Mitzenmacher, Michael},
  journal={arXiv preprint arXiv:2410.01035},
  year={2024}
}

@article{hu2024inference,
  title={Inference without interference: Disaggregate llm inference for mixed downstream workloads},
  author={Hu, Cunchen and Huang, Heyang and Xu, Liangliang and Chen, Xusheng and Xu, Jiang and Chen, Shuang and Feng, Hao and Wang, Chenxi and Wang, Sa and Bao, Yungang and others},
  journal={arXiv preprint arXiv:2401.11181},
  year={2024}
}

@article{schrage1968proof,
  title={A proof of the optimality of the shortest remaining processing time discipline},
  author={Schrage, Linus},
  journal={Operations Research},
  volume={16},
  number={3},
  pages={687--690},
  year={1968},
  publisher={INFORMS}
}

@article{tao2025prompt,
  title={Prompt-Aware Scheduling for Low-Latency LLM Serving},
  author={Tao, Yiheng and Zhang, Yihe and Dearing, Matthew T and Wang, Xin and Fan, Yuping and Lan, Zhiling},
  journal={arXiv preprint arXiv:2510.03243},
  year={2025}
}

@article{xie2026predicting,
  title={Predicting LLM output length via entropy-guided representations},
  author={Xie, Huanyi and Chen, Yubin and Wang, Liangyu and Hu, Lijie and Wang, Di},
  journal={arXiv preprint arXiv:2602.11812},
  year={2026}
}

@article{jaillet2025online,
  title={Online scheduling for llm inference with kv cache constraints},
  author={Jaillet, Patrick and Jiang, Jiashuo and Mellou, Konstantina and Molinaro, Marco and Podimata, Chara and Zhou, Zijie},
  journal={arXiv preprint arXiv:2502.07115},
  year={2025}
}

@article{chen2025adaptively,
  title={Adaptively robust llm inference optimization under prediction uncertainty},
  author={Chen, Zixi and Ye, Yinyu and Zhou, Zijie},
  journal={arXiv preprint arXiv:2508.14544},
  year={2025}
}

@article{wang2025llm,
  title={Llm serving optimization with variable prefill and decode lengths},
  author={Wang, Meixuan and Ye, Yinyu and Zhou, Zijie},
  journal={arXiv preprint arXiv:2508.06133},
  year={2025}
}

@article{ao2025optimizing,
  title={Optimizing llm inference: Fluid-guided online scheduling with memory constraints},
  author={Ao, Ruicheng and Luo, Gan and Simchi-Levi, David and Wang, Xinshang},
  journal={arXiv preprint arXiv:2504.11320},
  year={2025}
}

@article{liu2025greenllm,
  title={GreenLLM: SLO-Aware Dynamic Frequency Scaling for Energy-Efficient LLM Serving},
  author={Liu, Qunyou and Huang, Darong and Zapater, Marina and Atienza, David},
  journal={arXiv preprint arXiv:2508.16449},
  year={2025}
}

@article{pang2025optimizing,
  title={Optimizing LLM inference throughput via memory-aware and SLA-constrained dynamic batching},
  author={Pang, Bowen and Li, Kai and Wang, Feifan},
  journal={arXiv preprint arXiv:2503.05248},
  year={2025}
}

@inproceedings{li2024llm,
  title={Llm inference serving: Survey of recent advances and opportunities},
  author={Li, Baolin and Jiang, Yankai and Gadepally, Vijay and Tiwari, Devesh},
  booktitle={2024 IEEE High Performance Extreme Computing Conference (HPEC)},
  pages={1--8},
  year={2024},
  organization={IEEE}
}

@article{achiam2023gpt,
  title={Gpt-4 technical report},
  author={Achiam, Josh and Adler, Steven and Agarwal, Sandhini and Ahmad, Lama and Akkaya, Ilge and Aleman, Florencia Leoni and Almeida, Diogo and Altenschmidt, Janko and Altman, Sam and Anadkat, Shyamal and others},
  journal={arXiv preprint arXiv:2303.08774},
  year={2023}
}

@article{xiong2024search,
  title={When search engine services meet large language models: visions and challenges},
  author={Xiong, Haoyi and Bian, Jiang and Li, Yuchen and Li, Xuhong and Du, Mengnan and Wang, Shuaiqiang and Yin, Dawei and Helal, Sumi},
  journal={IEEE Transactions on Services Computing},
  volume={17},
  number={6},
  pages={4558--4577},
  year={2024},
  publisher={IEEE}
}

@article{liu2024deepseek,
  title={Deepseek-v3 technical report},
  author={Liu, Aixin and Feng, Bei and Xue, Bing and Wang, Bingxuan and Wu, Bochao and Lu, Chengda and Zhao, Chenggang and Deng, Chengqi and Zhang, Chenyu and Ruan, Chong and others},
  journal={arXiv preprint arXiv:2412.19437},
  year={2024}
}

@article{murali2024ai,
  title={Ai-assisted code authoring at scale: Fine-tuning, deploying, and mixed methods evaluation},
  author={Murali, Vijayaraghavan and Maddila, Chandra and Ahmad, Imad and Bolin, Michael and Cheng, Daniel and Ghorbani, Negar and Fernandez, Renuka and Nagappan, Nachiappan and Rigby, Peter C},
  journal={Proceedings of the ACM on Software Engineering},
  volume={1},
  number={FSE},
  pages={1066--1085},
  year={2024},
  publisher={ACM New York, NY, USA}
}

@inproceedings{zhen2025taming,
  title={Taming the titans: A survey of efficient llm inference serving},
  author={Zhen, Ranran and Li, Juntao and Ji, Yixin and Yang, Zhenlin and Liu, Tong and Xia, Qingrong and Duan, Xinyu and Wang, Zhefeng and Huai, Baoxing and Zhang, Min},
  booktitle={Proceedings of the 18th International Natural Language Generation Conference},
  pages={522--541},
  year={2025}
}

@techreport{smith1955various,
  title={Various optimizers for single-stage production},
  author={Smith, Wayne E},
  year={1955}
}

@article{grattafiori2024llama,
  title={The llama 3 herd of models},
  author={Grattafiori, Aaron and Dubey, Abhimanyu and Jauhri, Abhinav and Pandey, Abhinav and Kadian, Abhishek and Al-Dahle, Ahmad and Letman, Aiesha and Mathur, Akhil and Schelten, Alan and Vaughan, Alex and others},
  journal={arXiv preprint arXiv:2407.21783},
  year={2024}
}

@article{shoeybi2019megatron,
  title={Megatron-lm: Training multi-billion parameter language models using model parallelism},
  author={Shoeybi, Mohammad and Patwary, Mostofa and Puri, Raul and LeGresley, Patrick and Casper, Jared and Catanzaro, Bryan},
  journal={arXiv preprint arXiv:1909.08053},
  year={2019}
}

@misc{zheng2023lmsyschat1m,
      title={LMSYS-Chat-1M: A Large-Scale Real-World LLM Conversation Dataset}, 
      author={Lianmin Zheng and Wei-Lin Chiang and Ying Sheng and Tianle Li and Siyuan Zhuang and Zhanghao Wu and Yonghao Zhuang and Zhuohan Li and Zi Lin and Eric. P Xing and Joseph E. Gonzalez and Ion Stoica and Hao Zhang},
      year={2023},
      eprint={2309.11998},
      archivePrefix={arXiv},
      primaryClass={cs.CL}
}

@misc{bai2023longbench,
      title={LongBench: A Bilingual, Multitask Benchmark for Long Context Understanding}, 
      author={Yushi Bai and Xin Lv and Jiajie Zhang and Hongchang Lyu and Jiankai Tang and Zhidian Huang and Zhengxiao Du and Xiao Liu and Aohan Zeng and Lei Hou and Yuxiao Dong and Jie Tang and Juanzi Li},
      year={2023},
      eprint={2308.14508},
      archivePrefix={arXiv},
      primaryClass={cs.CL}
}

@inproceedings{devlin2019bert,
  title={Bert: Pre-training of deep bidirectional transformers for language understanding},
  author={Devlin, Jacob and Chang, Ming-Wei and Lee, Kenton and Toutanova, Kristina},
  booktitle={Proceedings of the 2019 conference of the North American chapter of the association for computational linguistics: human language technologies, volume 1 (long and short papers)},
  pages={4171--4186},
  year={2019}
}

@article{DBLP:journals/corr/abs-1908-08962,
  author    = {Iulia Turc and
               Ming{-}Wei Chang and
               Kenton Lee and
               Kristina Toutanova},
  title     = {Well-Read Students Learn Better: The Impact of Student Initialization
               on Knowledge Distillation},
  journal   = {CoRR},
  volume    = {abs/1908.08962},
  year      = {2019},
  url       = {http://arxiv.org/abs/1908.08962},
  eprinttype = {arXiv},
  eprint    = {1908.08962},
  bibsource = {dblp computer science bibliography, https://dblp.org}
}

@article{xiang2025servegen,
  title={Servegen: Workload characterization and generation of large language model serving in production},
  author={Xiang, Yuxing and Li, Xue and Qian, Kun and Yu, Wenyuan and Zhai, Ennan and Jin, Xin},
  journal={arXiv preprint arXiv:2505.09999},
  year={2025}
}

@misc{zhou2026positionllmservingneeds,
      title={Position: LLM Serving Needs Mathematical Optimization and Algorithmic Foundations, Not Just Heuristics}, 
      author={Zijie Zhou},
      year={2026},
      eprint={2605.01280},
      archivePrefix={arXiv},
      primaryClass={cs.DC},
      url={https://arxiv.org/abs/2605.01280}, 
}

@inproceedings{wang2025burstgpt,
  title={Burstgpt: A real-world workload dataset to optimize llm serving systems},
  author={Wang, Yuxin and Chen, Yuhan and Li, Zeyu and Kang, Xueze and Fang, Yuchu and Zhou, Yeju and Zheng, Yang and Tang, Zhenheng and He, Xin and Guo, Rui and others},
  booktitle={Proceedings of the 31st ACM SIGKDD Conference on Knowledge Discovery and Data Mining V. 2},
  pages={5831--5841},
  year={2025}
}

%%%%%%%%%%%%%%%%%%%%%%%%%%%%%%%%%%%%%%%%%%%%%%%%%%%%%%%%%%%%
\newpage
\appendix
\section{Related Works}
\label{related_work}
% LLM inference，本身分为system层次的改进和schedule算法/predictor的改进。
% schedule算法调研，主要涉及online schedule算法、1-bit排队论。以此说明LLM inference的独特challenges【不确定还需不需要，如果有第三部分的话】
% 基于zijie建模后的工作以及同期工作。
% 【Q：是否控制篇幅使得2、3部分不要超过第一部分，契合下community的taste】

% deep research一下（后续按survey改，然后参考下几篇强相关的文章的related work——今晚完成收工去写intro和实验！！

% 记得相关论点在intro叙述过的话，这里就换种叙述方式。只讲迭代
\textbf{LLM inference scheduling}.
Efficient serving of LLMs is pivotal for interactive applications. Early systems such as Orca~\cite{yu2022orca} and vLLM~\cite{kwon2023efficient} introduced iteration-level scheduling to optimize memory management. However, their reliance on FCFS policies often leads to severe head-of-line blocking under heterogeneous workloads~\cite{kaffes2019shinjuku, wu2023fast}. To mitigate this, FastServe~\cite{wu2023fast} employed multi-level feedback queues for dynamic prioritization, though frequent preemptions incur significant overhead in KV cache management. Recent prediction-based scheduling methods typically fall into three paradigms: formulating length prediction as a classification problem~\cite{jin2023s, stojkovic2025dynamollm, zheng2023response, hu2024inference, qiu2024power, shahout2024don}, adopting regression-based techniques to estimate output lengths~\cite{cheng2024enabling, qiu2024efficient, xie2026predicting}, or employing a Learning-to-Rank approach to evaluate relative request sizes~\cite{fu2024efficient, tao2025prompt}. Nevertheless, these prediction mechanisms are mainly integrated with time-centric heuristics to govern request prioritization and engine admittance.

\textbf{Theoretical foundations of memory-constrained scheduling}.
Classical scheduling theory has long established the optimality of time-centric heuristics—notably SJF and SRPT—for minimizing average latency in traditional single-machine environments~\cite{smith1955various, schrage1968proof}. However, recognizing the unique dynamics of LLM serving, recent research has begun to formalize its underlying mathematical modeling. \cite{jaillet2025online} established a foundational memory-bound scheduling model that strictly abstracts the autoregressive KV cache blow-up and continuous batching mechanisms. Based on this framework, \cite{wang2025llm} formally proved that SJF yields an unbounded worst-case approximation ratio. This theoretical breakthrough highlights the urgent need to move beyond traditional time-centric heuristics. 
% Recognized for accurately capturing the underlying mathematical dynamics between system throughput and out-of-memory (OOM) bottlenecks, 
Moreover, this modeling framework has inspired numerous derivatives~\cite{chen2025adaptively, ao2025optimizing, liu2025greenllm, pang2025optimizing}, fully proving its central role in reshaping the modern LLM inference ecosystem.

\section{Discussion on Why We Chose Greedy Algorithms}
% \subsection{Why Greedy Algorithms}
\label{discussion:greedy}
We deliberately design SVF as an explicit geometry-aware heuristic for the following reasons:

\textbf{Robustness Against Non-Stationary Traffic.} 
LLM traffic often exhibits unpredictable out-of-distribution (OOD) bursts and strong diurnal effects. Learning-based policies that rely on historical windows struggle to generalize to these sudden distribution shifts. In contrast, SVF decouples the \textit{prediction of request features} (which are semantically driven and relatively stationary) from the \textit{scheduling policy}. By proving that SVF maintains a constant worst-case approximation ratio, we ensure system stability regardless of traffic arrival dynamics.

\textbf{Complexity of 2D Spatio-Temporal Packing.} 
Continuous batching under KV cache limits is a highly stateful 2D geometric packing problem. A single large request casts a long ``memory shadow'' over future scheduling decisions, creating long-horizon dependencies that are notoriously difficult for gradient-based updates to capture reliably. SVF explicitly leverages the intrinsic geometry of the scheduling problem, providing a lightweight and system-safe solution.

\section{Broader Impacts}
\label{checklist10}
The primary contribution of our work is improving the system-level efficiency of LLM serving. On the positive side, our geometry-aware scheduling paradigm significantly reduces the latency and memory requirements for LLM inference. This democratization of computing resources can lower the operational costs and reduce the carbon footprint and energy consumption associated with large-scale AI deployments. On the negative side, by enhancing the throughput and reducing the cost of continuous batching, our system inherently makes it cheaper and faster to generate text at scale. Consequently, this could inadvertently lower the barrier for malicious actors to deploy LLMs for generating misinformation, spam, or harmful content. We encourage the community to deploy such high-throughput inference engines in tandem with robust content moderation and AI safety guardrails.

\section{Pseudocode of SVF and 1-bit SVF}
\label{code}

\begin{algorithm}[H]
\caption{Smallest Volume First (SVF) Scheduling}
\label{alg:svf}
\begin{algorithmic}[1]
\Require Request stream $\mathcal{R}$ arriving online
\Require Output length predictor $\mathcal{P}$ with max batch size $B_{max}$
\State \textbf{Initialize:} Prediction queue $Q_{pred} \leftarrow \emptyset$, Waiting queue $Q_{wait} \leftarrow \emptyset$
\State \quad On arrival of $r_i$ with prompt size $s_i$:
\State \quad Push $r_i$ to $Q_{pred}$

\Statex
\State \textbf{Process 1: Background Batched Predictor Worker}
\State \quad \textbf{while} True \textbf{do}
\State \quad \quad Wait until $Q_{pred}$ is not empty
\State \quad \quad $B_{req} \leftarrow $ Drain up to $B_{max}$ requests from $Q_{pred}$ \Comment{Greedy aggregation}
\State \quad \quad $\mathbf{\hat{o}} \leftarrow \mathcal{P}(\text{prompts of } B_{req})$ 
\State \quad \quad \textbf{for} each $r_i \in B_{req}$ and its prediction $\hat{o}_i \in \mathbf{\hat{o}}$ \textbf{do}
\State \quad \quad \quad $r_i.\text{priority} \leftarrow s_i \cdot \hat{o}_i + \frac{\hat{o}_i^2 + \hat{o}_i}{2}$ \Comment{Compute KV cache volume}
\State \quad \quad \quad Insert $r_i$ into $Q_{wait}$ in ascending order of $\text{priority}$
\State \quad \quad \textbf{end for}
\State \quad \textbf{end while}

\Statex
\State \textbf{Process 2: LLM Inference Engine} \Comment{System Native}
\State \quad \textbf{while} System KV cache has available capacity \textbf{do}
\State \quad \quad \textbf{if} $Q_{wait}$ is empty \textbf{then} \textbf{break}
\State \quad \quad $r_{next} \leftarrow \text{Pop}(Q_{wait})$ \Comment{Pops request with highest priority}
\State \quad \quad Add $r_{next}$ to running batch $Q_{run}$
\State \quad \textbf{end while}
\State \quad Execute one engine step (prefill or decode) for all requests in $Q_{run}$
\end{algorithmic}
\end{algorithm}

\begin{algorithm}[H]
\caption{1-bit Smallest Volume First (1-bit SVF) Scheduling}
\label{alg:onebit}
\begin{algorithmic}[1]
\Require Request stream $\mathcal{R}$ arriving online
\Require 1-bit length classifier $\mathcal{C}$ with max batch size $B_{max}$
\Require Global proxy lengths $O_0, O_1$
\State \textbf{Initialize:} Prediction queue $Q_{pred} \leftarrow \emptyset$, Waiting queue $Q_{wait} \leftarrow \emptyset$
\State \quad On arrival of $r_i$ with prompt size $s_i$:
\State \quad Push $r_i$ to $Q_{pred}$

\Statex
\State \textbf{Process 1: Background Batched Classifier Worker}
\State \quad \textbf{while} True \textbf{do}
\State \quad \quad Wait until $Q_{pred}$ is not empty
\State \quad \quad $B_{req} \leftarrow$ Drain up to $B_{max}$ requests from $Q_{pred}$ \Comment{Greedy aggregation}
\State \quad \quad $\mathbf{b} \leftarrow \mathcal{C}(\text{prompts of } B_{req})$ \Comment{Binary class: $0 = $ short, $1 = $ long}
\State \quad \quad \textbf{for} each $r_i \in B_{req}$ and its class $b_i \in \mathbf{b}$ \textbf{do}
\State \quad \quad \quad $r_i.\text{priority} \leftarrow s_i \cdot O_{b_i} + \frac{O_{b_i}^2 + O_{b_i}}{2}$
\State \quad \quad \quad Insert $r_i$ into $Q_{wait}$ in ascending order of $\text{priority}$
\State \quad \quad \textbf{end for}
\State \quad \textbf{end while}

\Statex
\State \textbf{Process 2: LLM Inference Engine} \Comment{System Native}
\State \quad \textbf{while} System KV cache has available capacity \textbf{do}
\State \quad \quad \textbf{if} $Q_{wait}$ is empty \textbf{then} \textbf{break}
\State \quad \quad $r_{next} \leftarrow \text{Pop}(Q_{wait})$ \Comment{Pops request with highest priority}
\State \quad \quad Add $r_{next}$ to running batch $Q_{run}$
\State \quad \textbf{end while}
\State \quad Execute one engine step (prefill or decode) for all requests in $Q_{run}$
\end{algorithmic}
\end{algorithm}

\section{Physical Validation of Normalized Volume}
\label{app:physical_volume_validation}

Following prior KV-constrained scheduling models~\cite{jaillet2025online}, we measure service in normalized token steps, providing a hardware-independent abstraction of the dynamically growing KV footprint. In a real serving engine, however, wall-clock iteration time varies with system state. We therefore examine whether the resulting logical Volume still preserves the physical service ordering relevant to scheduling by comparing it against a duration-weighted KV-time measurement obtained directly from vLLM.

\textbf{Experimental Design.}
We instrument vLLM v0.13.0 serving Meta-Llama-3.1-8B-Instruct on 200 LongBench requests under the FCFS policy with Poisson arrivals at 128 requests/s. At each engine iteration, we record the device-synchronized execution time $\Delta t_u$ and recover each request's physical KV allocation $b_{i,u}$ from the vLLM block tables, attributing shared prefix blocks evenly across their active references. For each request $i$, we compare the logical Volume $V_i^{\mathrm{log}}=s_i o_i+\frac{o_i(o_i+1)}{2}$ with the physical-service Volume $V_i^{\mathrm{phys}}=\sum_{u\in\mathcal{D}_i} b_{i,u}\Delta t_u,$ where $\mathcal{D}_i$ denotes the iterations in which request $i$ produces output tokens, and evaluate the agreement between the two induced request rankings.

% \begin{table}[t]
% \centering
% \small
% \caption{Physical validation of normalized Volume. The KV-time ratio compares
% the mean iteration duration between the highest- and lowest-KV quartiles;
% $\rho$ and $\tau_b$ measure rank agreement between logical Volume and the
% duration-weighted physical KV-time footprint.}
% \label{tab:key_abstraction_physical}
% \begin{tabular}{lr}
% \toprule
% Metric & Value \\
% \midrule
% Mixed iterations                         & 3,037 \\
% Iteration time, median [P5, P95]         & 39.09 [31.20, 52.97] ms \\
% Coefficient of variation                 & 0.189 \\
% High-/low-KV quartile time ratio         & $1.50\times$ \\
% Spearman $\rho$                          & 0.9868 \\
% Kendall $\tau_b$                         & 0.9069 \\
% \bottomrule
% \end{tabular}
% \end{table}

\textbf{Results.}
Physical execution departs from the normalized model along both dimensions of Volume. Spatially, physical KV occupancy varies across engine iterations: the highest- and lowest-KV quartiles have mean footprints of $0.0982$\,GiB and $0.0303$\,GiB, respectively, a $3.24\times$ difference. Temporally, iteration duration is likewise nonuniform, with the highest-KV quartile taking $1.50\times$ as long on average as the lowest-KV quartile. Despite these variations in both memory occupancy and wall-clock time, logical and physical Volume induce highly consistent request orderings: Spearman's $\rho$, measuring global rank correlation, reaches $\textbf{0.9868}$, while Kendall's $\tau_b$, measuring pairwise rank agreement, reaches $\textbf{0.9069}$. Only $\textbf{4.65\%}$ of request pairs are reversed. Thus, although the theoretical model normalizes both memory and time, logical Volume closely preserves the physical memory--time ordering relevant to SVF.

\section{Evaluation under Mixed Workloads}
\label{app:mixed_workload}

\begin{table}[t]
\centering
\small
\setlength{\tabcolsep}{3.8pt}
\caption{Performance comparison on the mixed workload under burst arrivals. Best non-oracle result per metric is \textbf{bold}.}
\label{tab:mixed_burst}
\begin{tabular}{lrrrrrr}
\toprule
Metric
& FCFS & SJF & 1-bit SVF & SVF & Oracle SJF & Oracle SVF \\
\midrule
Avg Lat. (s/tok) & 13.6163 & 14.4990 & 13.5875 & \textbf{8.7971} & 7.4813 & 3.1342 \\
P95 Lat. (s/tok) & 63.8118 & 67.3342 & \textbf{53.0479} & 55.2803 & 24.9506 & 19.3125 \\
Thpt. (tok/s) & 988.8 & 967.8 & 1009.5 & \textbf{1035.5} & -- & -- \\
\bottomrule
\end{tabular}
\end{table}

\begin{figure}[t]
    \centering
    \includegraphics[width=0.8\linewidth]{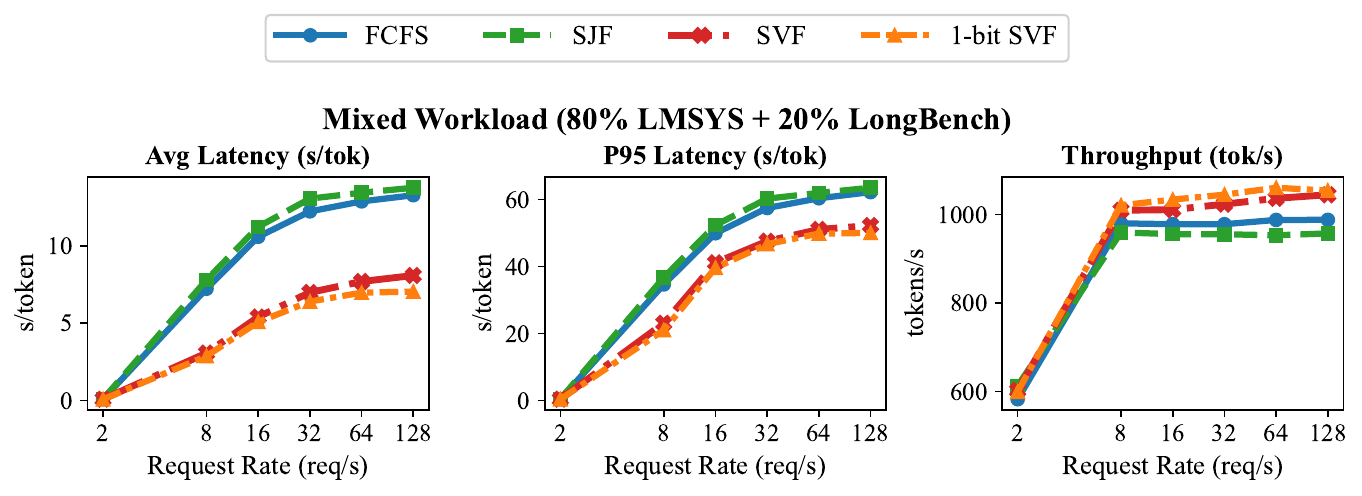}
    \caption{Practical performance on the mixed workload under Poisson arrivals.}
    \label{fig:mixed_workload}
\end{figure}

To further broaden the evaluation and assess robustness to workload composition, we consider a heterogeneous serving scenario in which requests with substantially different prompt--output geometries share the same queue. Specifically, we construct a mixed workload combining LMSYS and LongBench requests, allowing us to test whether the advantage of geometry-aware scheduling remains stable across a richer mixture of interactive and memory-intensive traffic.

\textbf{Setup.}
We construct a mixed workload containing 80\% LMSYS requests and 20\% LongBench requests, combining predominantly interactive traffic with a non-negligible fraction of memory-intensive long-context requests. We independently sample 2,000 requests for predictor training and 2,000 non-overlapping requests for serving evaluation. We use Meta-Llama-3.1-8B-Instruct and retain the same system configuration, predictor architectures, and evaluation metrics as in Section~\ref{EvaluationSetup}. We compare the four deployment-ready policies \{FCFS, SJF, SVF, 1-bit SVF\} under both Poisson arrivals and a burst arrival.

\textbf{Results.}
Figure~\ref{fig:mixed_workload} and Table~\ref{tab:mixed_burst} show that the advantage of geometry-aware scheduling persists under heterogeneous traffic. Across both Poisson and burst arrivals, SVF and 1-bit SVF consistently improve upon the SJF baseline once resource contention becomes significant, with SVF providing particularly strong gains in average latency and throughput and the lightweight 1-bit variant remaining highly competitive despite using only coarse output-length information. Together, these results highlight the consistent advantage of geometry-aware scheduling under heterogeneous traffic, supporting memory--time geometry as a robust scheduling primitive well aligned with the resource characteristics of LLM serving.

% \section{Controlled Target Replacement in PARS}
% \label{app:pars_volume}

% \textbf{Experimental design.}
% We conduct a controlled target-replacement study to test whether the geometry-aware ranking signal can be incorporated into an existing scheduler without modifying its surrounding pipeline. We retain the PARS predictor, pairwise-ranking loss, online admission mechanism, and starvation-prevention rule, and replace only the ranking
% target induced by output length $o_i$ with that induced by Volume $v_i=s_io_i+o_i(o_i+1)/2$. We evaluate Llama-3.1-8B on LongBench under burst and Poisson-128 arrivals, with 2,000 requests per run.

% \textbf{Results.}
% On LongBench, replacing the length-based target with Volume reduces PARS's average per-token latency by approximately 5.1\% under both burst and Poisson-128 arrivals. Since all other components of PARS are kept unchanged, this controlled result provides evidence that the geometry-aware ranking signal can be incorporated into an existing scheduler and improve mean latency in KV-intensive workloads.

\section{Detailed Proofs for Burst Arrival Scenario}
\label{sec:appendix_burst_proofs}

In this section, we provide the rigorous algebraic proofs for the burst arrival analysis presented in Section \ref{sec:burst}.

\subsection{Proof of Proposition \ref{prop:opt_lower_bound}}
\label{sec:appendix_proof_opt_lb}

\begingroup 
\def\theproposition{\ref{prop:opt_lower_bound}} 
\begin{proposition*}[Lower Bound of OPT, Restated] Let $vol_1 \le vol_2 \le \dots \le vol_N$ be the requests sorted by their geometric volumes. The total end-to-end latency of any offline optimal schedule (OPT) satisfies: 
\begin{equation}     
\text{TEL(OPT)} \ge \frac{1}{M} \sum_{j=1}^N \sum_{i=1}^j vol_i 
\end{equation} 
\end{proposition*} 
\endgroup

\begin{proof}
Consider the actual schedule generated by an omniscient offline optimal algorithm (OPT). Let the completion time sequence of the requests in its true execution be $C_1 \le C_2 \le \dots \le C_N$. Let the volume sequence corresponding to this completion order be $vol_{\pi(1)}, vol_{\pi(2)}, \dots, vol_{\pi(N)}$, where $\pi$ represents the specific permutation generated by the OPT schedule.

For any request completing at the $j$-th position, at its completion time $C_j$, the system must have completely processed all $j$ requests scheduled before or at this time. Due to the physical memory limits of the system, the maximum processing volume the system can provide within the time interval $[0, C_j]$ is strictly bounded by $C_j \cdot M$. This yields the inequality $C_j \cdot M \ge \sum_{i=1}^j vol_{\pi(i)}$, which can be rearranged as:
\begin{equation}
    C_j \ge \frac{1}{M} \sum_{i=1}^j vol_{\pi(i)}
\end{equation}

Summing the completion times over all $N$ requests gives the total end-to-end latency of OPT:
\begin{equation}
    \text{TEL(OPT)} = \sum_{j=1}^N C_j \ge \frac{1}{M} \sum_{j=1}^N \sum_{i=1}^j vol_{\pi(i)}
\end{equation}

Focusing on the right-hand side of the inequality, during the double summation, the volume of the first completed request $vol_{\pi(1)}$ is added exactly $N$ times; the second $vol_{\pi(2)}$ is added $N-1$ times, and so forth. This double summation is algebraically equivalent to a single summation with corresponding weights:
\begin{equation}
    \text{TEL(OPT)} \ge \frac{1}{M} \sum_{i=1}^N (N-i+1) \cdot vol_{\pi(i)}
\end{equation}

Still focusing on the right-hand side, the strategy $\pi$ that minimizes this weighted sum is to assign the largest weights to the smallest volumes. That is, the summation achieves its global minimum over the entire permutation space if and only if the sequence is sorted in strictly ascending order of their true volumes. Without loss of generality, we define $vol_1 \le vol_2 \le \dots \le vol_N$ (for requests with identical volumes, a unique arbitrary order is assigned to ensure strict indexing). Therefore, applying the rearrangement inequality, we obtain:
\begin{equation}
    \sum_{i=1}^N (N-i+1) \cdot vol_{\pi(i)} \ge \sum_{i=1}^N (N-i+1) \cdot vol_i = \sum_{j=1}^N \sum_{i=1}^j vol_i
\end{equation}

Ultimately, substituting this minimal sum back yields the tight lower bound for OPT:
\begin{equation}
    \text{TEL(OPT)} \ge \frac{1}{M} \sum_{j=1}^N \sum_{i=1}^j vol_i
\end{equation}
\end{proof}

% ==========================================================================
\subsection{Proof of Theorem \ref{thm:svf-burst}}
\label{sec:appendix_proof_svf_cr}

% --- 重述定理内容 --- 
\begingroup 
\def\thetheorem{\ref{thm:svf-burst}} 
\begin{theorem}[Worst-case Bound of SVF, Restated] 
Under the burst arrival model with $p_i = s_i + o_i \le \alpha M$,
SVF achieves a worst-case approximation ratio
\begin{equation}
  \frac{\text{TEL}(\text{SVF})}{ \text{TEL}(\text{OPT})} \le  1 + \frac{2}{1-\alpha}.
\end{equation}
\end{theorem} 
\addtocounter{theorem}{-1} 
\endgroup 
% ------------------

To establish the approximation ratio of SVF, we first introduce a second, natural lower bound for the offline optimal (OPT), which complements the area conservation bound derived in Proposition \ref{prop:opt_lower_bound}.

\begin{lemma}
\label{lemma:opt_lb_execution}
The total end-to-end latency of any offline optimal schedule (OPT) is strictly lower-bounded by the sum of the absolute minimum processing times of all requests:
\begin{equation}
    \text{TEL(OPT)} \ge \sum_{i=1}^N o_i
\end{equation}
\end{lemma}
\begin{proof}
In a non-preemptive discrete-time scheduling model, every request $i$ requires exactly $o_i$ consecutive active time steps to complete its decoding process. Thus, even if all requests could be processed instantaneously upon arrival without any queuing delay, the sum of their unavoidable execution times is $\sum_{i=1}^N o_i$. Any valid schedule must at least expend this baseline latency.
\end{proof}

\subsubsection{Algorithm Upper Bound}

For each request $i$, we define its specific metrics as follows:
\begin{itemize}
    \item \textbf{Peak Memory:} $p_i = s_i + o_i\le \alpha M$
    \item \textbf{Sorting Metric (Volume):} $vol_i = p_i \cdot o_i - \frac{o_i^2}{2} + \frac{o_i}{2}$
\end{itemize}

The SVF algorithm greedily selects requests to join the running batch in ascending order of their volume. Without loss of generality, we define a strict ordering:
\begin{equation}
    vol_1 \le vol_2 \le \dots \le vol_N
\end{equation}
For requests with identical volumes, an arbitrary but unique tie-breaking sequence is consistently applied to ensure strict indexing. If request $i$ strictly precedes request $j$ in this sequence, we denote it as $i \prec_\pi j$. The strict prefix set of requests processed before $j$ is defined as $\mathcal{P}_j = \{i \mid i \prec_\pi j\}$.

\begin{lemma}[Algorithmic Waiting Time of SVF]
\label{lemma:svf_wait_time}
In the discrete scheduling model under burst arrivals, the queuing time $W_j$ of any request $j$ is strictly bounded by the aggregated physical volume of its preceding requests: 
\begin{equation}
    W_j < \frac{2}{(1-\alpha)M}\sum_{i \in \mathcal{P}_j} vol_i
\end{equation}
\end{lemma}
\begin{proof}
% Consider any arbitrary discrete time step $t \in \{0, 1, \dots, W_j-1\}$ while request $j$ is lingering in the waiting queue. Let $A(t)$ denote the set of active requests concurrently executing in the GPU at time $t$. Due to the strict greedy nature of the SVF algorithm, it must hold that $A(t) \subseteq \mathcal{P}_j$. Furthermore, guided by the forward-looking memory check from \eqref{eq:c2}, if request $j$ were to be admitted at time $t$, there would exist some future time $t' \ge t$ where the aggregated memory exceeds the global capacity $M$.

Consider any arbitrary discrete time step $t \in \{0, 1, \dots, W_j-1\}$ while request $j$ is lingering in the waiting queue. Let $A(t)$ denote the set of active requests concurrently executing on the GPU at time $t$, and let $h(t)$ denote the head-of-line request in the waiting queue after the greedy admission phase at this time step. Since all requests arrive simultaneously and SVF follows the fixed priority ordering $\pi$ without bypassing a blocked higher-priority request, request $j$ remaining in the waiting queue implies
\begin{equation}
    h(t) \in \mathcal{P}_j \cup \{j\},
\end{equation}
and every request already admitted before $j$ must precede $j$ under $\pi$. Therefore, it holds that
\begin{equation}
    A(t) \subseteq \mathcal{P}_j.
\end{equation}
Furthermore, guided by the forward-looking memory check from \eqref{eq:c2}, since $h(t)$ is the current head-of-line request and remains blocked, if $h(t)$ were to be admitted at time $t$, there would exist some future time $t' \ge t$ where the aggregated memory exceeds the global capacity $M$.

For any request $i$, its dynamic memory footprint $m_i(\tau)$ throughout its lifecycle satisfies $m_i(\tau) \le p_i$. Consequently, the maximum possible memory sum of the active set $A(t)$ at any future point will not exceed $\sum_{i \in A(t)} p_i$. 

Suppose, for the sake of contradiction, that $\sum_{i \in A(t)} p_i \le (1-\alpha) M$. Even if the blocked head-of-line request $h(t)$ reached its peak memory $p_{h(t)}$, the system's total memory usage would be bounded by:
\begin{equation}
    \sum_{i \in A(t)} p_i + p_{h(t)}
    \le (1-\alpha)M + \alpha M = M,
\end{equation}
where the inequality follows from the peak-memory assumption $p_i \le \alpha M$ for every request $i$. This contradicts the premise of the forward-looking memory violation. Therefore, for every single time step $t$ during $j$'s waiting period, the system guarantees a minimum memory utilization rate:
\begin{equation}
    \sum_{i \in A(t)} p_i > (1-\alpha)M.
\end{equation}

Integrating this inequality over every discrete time step within $j$'s entire waiting period $[0, W_j-1]$, we obtain:
\begin{equation}
    \sum_{t=0}^{W_j-1} \sum_{i \in A(t)} p_i > \sum_{t=0}^{W_j-1} (1-\alpha)M = (1-\alpha)MW_j
\end{equation}

The left-hand side sums over time first, then iterates through active requests. We can mathematically swap this summation order to first enumerate the requests, then calculate each request's active duration. Since $A(t) \subseteq \mathcal{P}_j$, any request $i \in \mathcal{P}_j$ remains active in the system for at most $o_i$ steps. Hence, the peak $p_i$ of request $i$ is accumulated at most $o_i$ times in the double summation:
\begin{equation}
    \sum_{t=0}^{W_j-1} \sum_{i \in A(t)} p_i \le \sum_{i \in \mathcal{P}_j} p_i \cdot o_i
\end{equation}

Substituting this upper bound back yields:
\begin{equation}
    \sum_{i \in \mathcal{P}_j} p_i \cdot o_i > (1-\alpha)MW_j
\end{equation}

Leveraging the intrinsic geometric property of our defined volume, it inherently satisfies $p_i \cdot o_i < 2 \cdot vol_i$. Applying this property establishes the final strict bound:
\begin{equation}
    (1-\alpha)MW_j < \sum_{i \in \mathcal{P}_j} p_i \cdot o_i < 2 \sum_{i \in \mathcal{P}_j} vol_i
\end{equation}
Dividing both sides by $(1-\alpha)M$ concludes the proof of the lemma:
\begin{equation}
    W_j < \frac{2}{(1-\alpha)M} \sum_{i \in \mathcal{P}_j} vol_i
\end{equation}
\end{proof}

\subsubsection{Final Alignment of the Approximation Ratio}

Drawing upon Lemma \ref{lemma:svf_wait_time}, we have $W_j < \frac{2}{(1-\alpha)M}\sum_{i \prec_\pi j} vol_i$. By definition, the total end-to-end latency of our algorithm\ is the sum of queuing delays and execution times across all requests:
\begin{equation}
    \text{TEL(SVF)} = \sum_{j=1}^N (W_j + o_j) \le \frac{2}{(1-\alpha)M}\sum_{j=1}^N \sum_{i \prec_\pi j} vol_i + \sum_{j=1}^N o_j
\end{equation}

To tightly align this upper bound with the offline optimal, we enforce the identical arbitrary tie-breaking strategy for equally prioritized volumes in both SVF and OPT. % This consistency ensures that the structural double summation precisely mirrors the OPT bound. 
Applying the area conservation bound (Proposition \ref{prop:opt_lower_bound}) to the first term, and the natural execution time bound (Lemma \ref{lemma:opt_lb_execution}) to the second term, we arrive at:
\begin{equation}
    \text{TEL(SVF)} \le \frac{2}{1-\alpha} \cdot \text{TEL(OPT)} + 1 \cdot \text{TEL(OPT)} = (\frac{2}{1-\alpha}+1) \cdot \text{TEL(OPT)}
\end{equation}

% ==========================================================================
\subsection{Proof of Theorem \ref{thm:svf_1bit_burst}}
\label{sec:appendix_proof_1bit}

\begingroup
\def\thetheorem{\ref{thm:svf_1bit_burst}}
\begin{theorem}[Worst-case Bound of 1-bit SVF, Restated]
Under the burst arrival model with $s+o\le T$, 1-bit SVF algorithm configures a binary classification threshold $\theta = \sqrt{T}$ and assigns proxy lengths $O_0 = T^{1/4}$ for short requests ($o \le \theta$) and $O_1 = T^{3/4}$ for long requests ($o > \theta$), achieving an approximation ratio strictly bounded by $\frac{ \text{TEL}(\text{1-bit SVF})}{ \text{TEL}(\text{OPT})}= \mathcal{O}(T)$.
\end{theorem}
\addtocounter{theorem}{-1}
\endgroup

\begin{proof}
The true output length $o_i$ is bounded by the maximum model length $T$, i.e., $o_i \in [1, T]$.% To minimize the maximum possible error (Minimax) without any prior distribution, we adopt a geometric mean proxy mapping.

\textbf{Step 1: Minimax Geometric Mean Proxy Design.}
For any uncertain parameter bounded in an interval $[L, U]$, the single proxy value that minimizes the maximum ratio distortion is its geometric mean $\sqrt{L \cdot U}$. We partition the entire output space using the threshold $\theta = \sqrt{T}$, creating two classes: $\mathcal{S} = [1, \theta]$ and $\mathcal{L} = (\theta, T]$. 

The proxy outputs are defined geometrically for each class:
\begin{itemize}
    \item Short Class ($\mathcal{S}$): $O_0 = \sqrt{1 \cdot \theta} = \sqrt{\theta} = T^{1/4}$
    \item Long Class ($\mathcal{L}$): $O_1 = \sqrt{\theta \cdot T} = T^{3/4}$
\end{itemize}
Since $T \ge 1$, we have $b=T^{1/4} \ge 1$. Thus, for any request $i$, its true length $o_i$ and proxy $O_m$ strictly satisfy the inequality:
\begin{equation}
\label{eq:length_distortion}
    \frac{1}{b} O_m \le o_i \le b O_m
\end{equation}

\textbf{Step 2: Strict Absolute Volume Bounding.}
The 1-bit SVF algorithm dynamically schedules requests by sorting their proxy volumes: $\hat{v}_i = s_i O_m + \frac{O_m^2}{2} + \frac{O_m}{2}$. We must rigorously bound the true volume $vol_i = s_i o_i + \frac{o_i^2}{2} + \frac{o_i}{2}$ against this proxy.

Applying $o_i \le b O_m$ to the true volume yields:
\begin{equation}
    vol_i \le s_i (b O_m) + \frac{(b O_m)^2}{2} + \frac{b O_m}{2}
\end{equation}
Since $b \ge 1$, it trivially holds that $b \le b^2$. Substituting $b^2$ for all linear $b$ terms provides a strict algebraic upper bound:
\begin{equation}
    vol_i \le b^2 (s_i O_m) + b^2 \frac{O_m^2}{2} + b^2 \frac{O_m}{2} = b^2 \left(s_i O_m + \frac{O_m^2}{2} + \frac{O_m}{2} \right) = b^2 \hat{v}_i
\end{equation}

Similarly, applying $o_i \ge \frac{O_m}{b}$ to the true volume gives:
\begin{equation}
    vol_i \ge s_i \left(\frac{O_m}{b}\right) + \frac{(O_m/b)^2}{2} + \frac{O_m/b}{2}
\end{equation}
Since $b \ge 1$, it holds that $\frac{1}{b} \ge \frac{1}{b^2}$. Substituting $\frac{1}{b^2}$ for all linear $\frac{1}{b}$ terms yields:
\begin{equation}
    vol_i \ge \frac{1}{b^2} (s_i O_m) + \frac{1}{b^2} \frac{O_m^2}{2} + \frac{1}{b^2} \frac{O_m}{2} = \frac{1}{b^2} \hat{v}_i
\end{equation}

Combining these, we establish a universal, class-agnostic absolute volume inequality:
\begin{equation}
\label{eq:volume_distortion}
    \frac{1}{b^2} \hat{v}_i \le vol_i \le b^2 \hat{v}_i
\end{equation}

\textbf{Step 3: The Universal Alignment.}
Suppose the 1-bit SVF algorithm prioritizes request $y$ over request $x$. This implies that the algorithm evaluated $\hat{v}_y \le \hat{v}_x$. We aim to bound the maximum possible inversion ratio $\rho = \frac{vol_y}{vol_x}$ of their true volumes.

Applying the absolute volume inequalities \eqref{eq:volume_distortion} successively:
\begin{equation}
    vol_y \le b^2 \hat{v}_y \le b^2 \hat{v}_x \le b^2 (b^2 vol_x) = b^4 vol_x
\end{equation}
Because our minimax design guarantees $b = T^{1/4}$, we exactly have $b^4 = T$. Therefore, for any pair of requests scheduled in proxy order, their true volume ratio is strictly bounded by:
\begin{equation}
    \frac{vol_y}{vol_x} \le T
\end{equation}

\textbf{Step 4: Waiting-Time Certificate.}
Let $\pi$ denote the total ordering induced by the proxy volume $\hat v$, with a fixed tie-breaking rule, and define $P_j := \{i : i \prec_\pi j\}$ as the set of requests preceding request $j$ under the 1-bit SVF ordering. Let $W_j$ denote the queuing time of request $j$, and recall that $p_j = s_j + o_j$ denotes its peak memory footprint.

Consider any time step $t$ while request $j$ remains in the waiting queue. Let $h(t)$ denote the head-of-line request in the waiting queue after the greedy admission phase at time $t$, and let $A(t)$ denote the set of active requests concurrently executing on the GPU. Since all requests arrive simultaneously and 1-bit SVF follows the fixed priority ordering $\pi$ without bypassing a blocked higher-priority request, we have
\begin{equation}
    h(t) \in P_j \cup \{j\},
    \qquad
    A(t) \subseteq P_j.
\label{eq:1bit_hol_prefix}
\end{equation}

Furthermore, since $h(t)$ is the current head-of-line request and remains blocked, the forward-looking memory check implies that admitting $h(t)$ would induce a future memory violation. For any active request $i$, its dynamic memory footprint is always upper-bounded by its peak memory $p_i$. Therefore,
\begin{equation}
    \sum_{i\in A(t)} p_i
    >
    M-p_{h(t)}.
\label{eq:1bit_hol_memory}
\end{equation}

We now partition the analysis into two cases according to the peak memory $p_j$ of the target request.

\textit{Case 1: $p_j \le M/3$.}
We further partition the waiting steps of request $j$ according to the peak memory of the corresponding head-of-line request:
\begin{equation}
    \mathcal{T}_j^{\mathrm{S}}
    :=
    \{t<W_j : p_{h(t)}\le M/3\},
    \qquad
    \mathcal{T}_j^{\mathrm{L}}
    :=
    \{t<W_j : p_{h(t)}>M/3\}.
\label{eq:1bit_wait_partition}
\end{equation}
Clearly,
\begin{equation}
    W_j
    =
    |\mathcal{T}_j^{\mathrm{S}}|
    +
    |\mathcal{T}_j^{\mathrm{L}}|.
\label{eq:1bit_wait_decompose}
\end{equation}

We first bound the waiting steps in $\mathcal{T}_j^{\mathrm{S}}$. For every $t\in\mathcal{T}_j^{\mathrm{S}}$, equation~\eqref{eq:1bit_hol_memory} gives
\begin{equation}
    \sum_{i\in A(t)} p_i
    >
    M-p_{h(t)}
    \ge
    \frac{2M}{3}.
\label{eq:1bit_small_hol_memory}
\end{equation}

Summing over all time steps in $\mathcal{T}_j^{\mathrm{S}}$ yields
\begin{equation}
    \frac{2M}{3}
    |\mathcal{T}_j^{\mathrm{S}}|
    <
    \sum_{t\in\mathcal{T}_j^{\mathrm{S}}}
    \sum_{i\in A(t)}p_i.
\label{eq:1bit_small_hol_integral}
\end{equation}

Since $A(t)\subseteq P_j$ and each predecessor $i\in P_j$ remains active for at most $o_i$ decoding steps,
\begin{equation}
    \sum_{t\in\mathcal{T}_j^{\mathrm{S}}}
    \sum_{i\in A(t)}p_i
    \le
    \sum_{i\in P_j}p_io_i.
\label{eq:1bit_small_hol_swap}
\end{equation}

By the definition of geometric volume,
\begin{equation}
    p_io_i=(s_i+o_i)o_i<2\,\mathrm{vol}_i.
\end{equation}
Combining this with equations~\eqref{eq:1bit_small_hol_integral} and~\eqref{eq:1bit_small_hol_swap}, we obtain
\begin{equation}
    |\mathcal{T}_j^{\mathrm{S}}|
    <
    \frac{3}{M}
    \sum_{i\in P_j}\mathrm{vol}_i.
\label{eq:1bit_small_hol_volume}
\end{equation}

From Step 3, every predecessor $i\in P_j$ satisfies
\begin{equation}
    \mathrm{vol}_i
    \le
    T\min\{\mathrm{vol}_i,\mathrm{vol}_j\}.
\end{equation}
Therefore,
\begin{equation}
    |\mathcal{T}_j^{\mathrm{S}}|
    <
    \frac{3T}{M}
    \sum_{i\in P_j}
    \min\{\mathrm{vol}_i,\mathrm{vol}_j\}.
\label{eq:1bit_small_hol_final}
\end{equation}

We next bound the waiting steps in $\mathcal{T}_j^{\mathrm{L}}$. Since $p_j\le M/3$, any head-of-line request $h(t)$ satisfying $p_{h(t)}>M/3$ must be a strict predecessor of $j$, i.e.,
\begin{equation}
    h(t)\in P_j.
\end{equation}

For every predecessor $h\in P_j$ with $p_h>M/3$, let $\ell_h$ denote the number of time steps during which $h$ is the blocked head-of-line request while $j$ is still waiting, with $\ell_h=0$ if $h$ never becomes the blocked head-of-line request. Because the ordering is fixed and no request can bypass a blocked head-of-line request, each request can contribute at most one such contiguous blocking interval. Hence,
\begin{equation}
    |\mathcal{T}_j^{\mathrm{L}}|
    =
    \sum_{\substack{h\in P_j\\ p_h>M/3}}
    \ell_h.
\label{eq:1bit_large_hol_sum}
\end{equation}

We claim that every such large head-of-line predecessor satisfies
\begin{equation}
    \ell_h
    \le
    \frac{6T}{M}
    \min\{\mathrm{vol}_h,\mathrm{vol}_j\}.
\label{eq:1bit_large_hol_claim}
\end{equation}

To establish this claim, first observe that
\begin{equation}
    \mathrm{vol}_h
    \ge
    p_h
    >
    \frac{M}{3},
\label{eq:1bit_large_hol_vol}
\end{equation}
where $\mathrm{vol}_h\ge p_h$ follows from $o_h\ge1$. Moreover, since $p_h\le T$, we have
\begin{equation}
    M<3T.
\label{eq:1bit_large_implies_M}
\end{equation}

We also use the following consequence of the proxy construction. Suppose a long request $\ell$ precedes a short request $x$ under $\pi$. Let
\begin{equation}
    a:=T^{1/4},
    \qquad
    O_0=a,
    \qquad
    O_1=a^3.
\end{equation}
Since $\hat v_\ell\le\hat v_x$, the minimum possible proxy volume of the long request must not exceed the proxy volume of $x$, and therefore
\begin{equation}
    s_xO_0+\frac{O_0^2+O_0}{2}
    \ge
    \frac{O_1^2+O_1}{2}.
\end{equation}
Thus,
\begin{equation}
    s_x
    \ge
    \frac{a^6+a^3-a^2-a}{2a}
    =
    \frac{a^5+a^2-a-1}{2}.
\end{equation}
Since $o_x\ge1$,
\begin{align}
    \mathrm{vol}_x
    &=
    s_xo_x+\frac{o_x^2+o_x}{2}
    \nonumber\\
    &\ge
    s_x+1
    \nonumber\\
    &\ge
    \frac{a^5+a^2-a+1}{2}
    \nonumber\\
    &\ge
    \frac{a^4}{2}
    =
    \frac{T}{2},
\label{eq:1bit_long_before_short}
\end{align}
where the last inequality follows from
\begin{equation}
    a^5+a^2-a+1-a^4
    =
    (a-1)(a^4+a)+1
    \ge0.
\end{equation}

We now return to the blocking interval of request $h$. Since no lower-priority request can bypass $h$, once $h$ becomes the head-of-line request, no new request can be admitted ahead of it. Every request active at that moment requires at most $T$ total decoding steps. Hence all currently active predecessors must finish within at most $T$ additional steps, at which point $h$, being individually feasible, can be admitted. Therefore,
\begin{equation}
    \ell_h\le T.
\label{eq:1bit_large_hol_duration}
\end{equation}

If request $j$ belongs to the long class, then $o_j>\sqrt{T}$ and therefore
\begin{equation}
    \mathrm{vol}_j
    \ge
    \frac{o_j^2+o_j}{2}
    >
    \frac{T}{2}
    >
    \frac{M}{6},
\end{equation}
where the last inequality uses $M<3T$.

Likewise, if $j$ is short but either $h$ itself is long or some long request precedes $h$, then a long request precedes $j$ under $\pi$. Equation~\eqref{eq:1bit_long_before_short} again gives
\begin{equation}
    \mathrm{vol}_j
    \ge
    \frac{T}{2}
    >
    \frac{M}{6}.
\end{equation}

Together with equation~\eqref{eq:1bit_large_hol_vol}, these cases imply
\begin{equation}
    \min\{\mathrm{vol}_h,\mathrm{vol}_j\}
    >
    \frac{M}{6}.
\end{equation}
Combining this with equation~\eqref{eq:1bit_large_hol_duration} yields
\begin{equation}
    \ell_h
    \le
    T
    <
    \frac{6T}{M}
    \min\{\mathrm{vol}_h,\mathrm{vol}_j\}.
\label{eq:1bit_large_hol_regular}
\end{equation}

It remains only to consider the case in which $j$, $h$, and all predecessors of $h$ are short requests. Since $h$ and $j$ use the same proxy length $O_0$, the ordering $h\prec_\pi j$ implies
\begin{equation}
    s_h\le s_j.
\end{equation}
Because $p_h=s_h+o_h>M/3$ and $o_h\le\sqrt{T}$,
\begin{equation}
    s_j
    \ge
    s_h
    >
    \frac{M}{3}-\sqrt{T}.
\end{equation}

If $M\ge6\sqrt{T}$, then
\begin{equation}
    \mathrm{vol}_j
    \ge
    s_j+1
    >
    \frac{M}{3}-\sqrt{T}
    \ge
    \frac{M}{6}.
\end{equation}
Together with $\mathrm{vol}_h>M/3$, equation~\eqref{eq:1bit_large_hol_regular} again follows.

Otherwise, $M<6\sqrt{T}$. Since all predecessors of $h$ are short, every request active when $h$ becomes head-of-line has at most $\sqrt{T}$ total decoding steps. No new request can bypass $h$, so all of these active predecessors finish within at most $\sqrt{T}$ steps. Therefore,
\begin{equation}
    \ell_h\le\sqrt{T}.
\end{equation}
Moreover, every request has volume at least one, so
\begin{equation}
    \min\{\mathrm{vol}_h,\mathrm{vol}_j\}\ge1.
\end{equation}
Using $M<6\sqrt{T}$,
\begin{equation}
    \frac{6T}{M}
    \min\{\mathrm{vol}_h,\mathrm{vol}_j\}
    >
    \sqrt{T}
    \ge
    \ell_h.
\end{equation}

Thus equation~\eqref{eq:1bit_large_hol_claim} holds in all cases.

Summing equation~\eqref{eq:1bit_large_hol_claim} over all large head-of-line predecessors gives
\begin{align}
    |\mathcal{T}_j^{\mathrm{L}}|
    &=
    \sum_{\substack{h\in P_j\\p_h>M/3}}\ell_h
    \nonumber\\
    &\le
    \frac{6T}{M}
    \sum_{\substack{h\in P_j\\p_h>M/3}}
    \min\{\mathrm{vol}_h,\mathrm{vol}_j\}
    \nonumber\\
    &\le
    \frac{6T}{M}
    \sum_{i\in P_j}
    \min\{\mathrm{vol}_i,\mathrm{vol}_j\}.
\label{eq:1bit_large_hol_final}
\end{align}

Combining equations~\eqref{eq:1bit_wait_decompose}, \eqref{eq:1bit_small_hol_final}, and~\eqref{eq:1bit_large_hol_final}, we obtain
\begin{equation}
    W_j
    \le
    \frac{9T}{M}
    \sum_{i\in P_j}
    \min\{\mathrm{vol}_i,\mathrm{vol}_j\}.
\label{eq:1bit_case1_final}
\end{equation}

\textit{Case 2: $p_j > M/3$.}
While request $j$ is waiting, at least one predecessor in $P_j$ must be active after the greedy admission phase. Otherwise, since request $j$ is individually feasible, the scheduler could continue admitting requests according to $\pi$ and could not leave $j$ waiting with no active predecessor. Hence, the waiting period of $j$ is contained in the union of the active lifetimes of its predecessors, giving
\begin{equation}
    W_j
    \le
    \sum_{i\in P_j}o_i.
\label{eq:1bit_case2_wait}
\end{equation}

Since $p_j>M/3$ and $p_j\le T$, we have $M<3T$.
We next bound each $o_i$ in equation~\eqref{eq:1bit_case2_wait} by the pairwise minimum volume.

If $\mathrm{vol}_i\le\mathrm{vol}_j$, then
$o_i\le\mathrm{vol}_i=\min\{\mathrm{vol}_i,\mathrm{vol}_j\}$.
Since $M<3T$, it follows that
\begin{equation}
    o_i
    <
    \frac{3T}{M}
    \min\{\mathrm{vol}_i,\mathrm{vol}_j\}.
\label{eq:1bit_case2_small}
\end{equation}

Otherwise, if $\mathrm{vol}_i>\mathrm{vol}_j$, then
$\min\{\mathrm{vol}_i,\mathrm{vol}_j\}=\mathrm{vol}_j$.
Moreover, since $o_j\ge1$,
\begin{equation}
    \mathrm{vol}_j
    =
    s_jo_j+\frac{o_j^2+o_j}{2}
    \ge
    s_j+o_j
    =
    p_j
    >
    \frac{M}{3}.
\end{equation}
Thus,
\begin{equation}
    \frac{3T}{M}
    \min\{\mathrm{vol}_i,\mathrm{vol}_j\}
    =
    \frac{3T}{M}\mathrm{vol}_j
    >
    T
    \ge
    o_i.
\label{eq:1bit_case2_large}
\end{equation}

Combining equations~\eqref{eq:1bit_case2_small} and~\eqref{eq:1bit_case2_large}, for every predecessor $i\in P_j$ we have
\begin{equation}
    o_i
    \le
    \frac{3T}{M}
    \min\{\mathrm{vol}_i,\mathrm{vol}_j\}.
\label{eq:1bit_case2_pair}
\end{equation}

Substituting equation~\eqref{eq:1bit_case2_pair} into equation~\eqref{eq:1bit_case2_wait} yields
\begin{equation}
    W_j
    \le
    \frac{3T}{M}
    \sum_{i\in P_j}
    \min\{\mathrm{vol}_i,\mathrm{vol}_j\}.
\label{eq:1bit_case2_final}
\end{equation}

Combining the two cases, and using the larger coefficient from Case 1, we obtain the following unified waiting-time bound for every request $j$:
\begin{equation}
    W_j
    \le
    \frac{9T}{M}
    \sum_{i\in P_j}
    \min\{\mathrm{vol}_i,\mathrm{vol}_j\}.
\label{eq:1bit_unified_wait}
\end{equation}

\textbf{Step 5: Final Approximation Ratio.}
Summing equation~\eqref{eq:1bit_unified_wait} over all requests gives
\begin{equation}
    \sum_{j=1}^{N} W_j
    \le
    \frac{9T}{M}
    \sum_{j=1}^{N}
    \sum_{i\in P_j}
    \min\{\mathrm{vol}_i,\mathrm{vol}_j\}.
\label{eq:1bit_global_wait}
\end{equation}

Because $\pi$ is a total ordering, every unordered pair of distinct requests $\{i,j\}$ appears exactly once in the double sum: either $i\in P_j$ or $j\in P_i$.
Therefore,
\begin{equation}
    \sum_{j=1}^{N}
    \sum_{i\in P_j}
    \min\{\mathrm{vol}_i,\mathrm{vol}_j\}
    =
    \sum_{\{i,j\}}
    \min\{\mathrm{vol}_i,\mathrm{vol}_j\}.
\label{eq:1bit_pair_identity}
\end{equation}

Let
$\mathrm{vol}_{(1)}\le\mathrm{vol}_{(2)}\le\cdots\le\mathrm{vol}_{(N)}$
denote the true request volumes sorted in ascending order.
The pairwise minimum sum can then be written as
\begin{equation}
    \sum_{\{i,j\}}
    \min\{\mathrm{vol}_i,\mathrm{vol}_j\}
    =
    \sum_{j=1}^{N}
    \sum_{i=1}^{j-1}
    \mathrm{vol}_{(i)}.
\label{eq:1bit_pair_prefix}
\end{equation}

By the Area Conservation Lower Bound in Proposition~\ref{prop:opt_lower_bound},
\begin{equation}
    \mathrm{TEL}(\mathrm{OPT})
    \ge
    \frac{1}{M}
    \sum_{j=1}^{N}
    \sum_{i=1}^{j}
    \mathrm{vol}_{(i)}.
\label{eq:1bit_opt_lower}
\end{equation}

Since
\begin{equation}
    \sum_{j=1}^{N}
    \sum_{i=1}^{j-1}
    \mathrm{vol}_{(i)}
    \le
    \sum_{j=1}^{N}
    \sum_{i=1}^{j}
    \mathrm{vol}_{(i)},
\end{equation}
we obtain
\begin{equation}
    \sum_{j=1}^{N}W_j
    \le
    9T\cdot\mathrm{TEL}(\mathrm{OPT}).
\label{eq:1bit_wait_opt}
\end{equation}

Finally, the total end-to-end latency of 1-bit SVF is
\begin{equation}
    \mathrm{TEL}(1\text{-bit SVF})
    =
    \sum_{j=1}^{N}(W_j+o_j).
\label{eq:1bit_tel}
\end{equation}

By the Natural Execution Time Bound,
\begin{equation}
    \sum_{j=1}^{N}o_j
    \le
    \mathrm{TEL}(\mathrm{OPT}).
\end{equation}
Combining this with equation~\eqref{eq:1bit_wait_opt}, we obtain
\begin{align}
    \mathrm{TEL}(1\text{-bit SVF})
    &\le
    9T\cdot\mathrm{TEL}(\mathrm{OPT})
    +
    \mathrm{TEL}(\mathrm{OPT})
    \nonumber\\
    &=
    (9T+1)\mathrm{TEL}(\mathrm{OPT}).
\label{eq:1bit_final_tel}
\end{align}

Therefore,
\begin{equation}
    \frac{\mathrm{TEL}(1\text{-bit SVF})}
    {\mathrm{TEL}(\mathrm{OPT})}
    \le
    9T+1
    =
    O(T),
\label{eq:1bit_final_cr}
\end{equation}
which completes the proof.
\end{proof}
\section{Proofs for Stochastic Analysis under Poisson Arrivals}
\label{sec:appendix_proof_poisson}

\textbf{Scope of the stochastic model.} 
To evaluate the steady-state expected performance under Poisson arrivals, we first establish the physical capacity limit of the offline optimal (OPT) scheduling by constructing a precise 2D volumetric queuing model.

Our stochastic analysis uses an untruncated geometric distribution $o_i\sim\mathrm{Geo}(\mu)$ as an analytically tractable reference model for output lengths. The geometric reference model and the physical feasibility condition
$s_i+o_i\leq\alpha M$ play distinct roles: the geometric model captures the token-by-token stochastic generation process and, in particular, provides the memoryless property used in the residual-volume analysis, whereas $s_i+o_i\le \alpha M$ enforces the finite KV-cache capacity of the serving system. Strictly speaking, imposing this feasibility condition corresponds to truncating the geometric tail at the maximum admissible output length $\alpha M-s_i$. We use the untruncated geometric distribution as an analytically convenient approximation when the truncation threshold lies sufficiently far in the tail, so that the excluded probability mass $\Pr(o_i>\alpha M-s_i)=(1-\mu)^{\alpha M-s_i}$ is negligible.

For reference, \cite{fu2024efficient} report a mean output length of
293 tokens for Llama-3-8B on LMSYS-Chat-1M. Calibrating our geometric model gives $1/\mu=293$. Under the maximum generation budget of $4{,}096$
tokens used in our experiments, the corresponding tail probability is $ \Pr(o_i>4096)\approx 8.29\times10^{-7}$. This provides an illustrative sanity check that, at generation-length scales observed in prior LLM-serving workloads, the omitted geometric
tail can be very small.

\subsection{Lower Bound of OPT}

\begin{lemma}
\label{lemma:opt_poisson_lb}
Let $W_{exact} = \lambda \frac{1-\mu}{\mu^2} \left( \mathbb{E}[s] + \frac{2}{\mu} \right)$. The expected end-to-end latency of the offline optimal schedule is lower-bounded by:
\begin{equation}
\mathbb{E}[TEL(\text{OPT})] \ge \max \left\{ \frac{1}{\mu}, \frac{W_{exact}}{M} \right\}
\end{equation}
\end{lemma}
\begin{proof}
By the natural execution time limit, no algorithm can bypass the autoregressive generation time, thus $\mathbb{E}[TEL(\text{OPT})] \ge \mathbb{E}[o] = \frac{1}{\mu}$.

For the queuing capacity bound, we rely on the core geometric principle: $V_{consumed} < V_{available}$. Consider the discrete time steps $\tau \in \{1, 2, \dots, T\}$. We define two global cumulative sequences: the arrival volume $A(\tau)$ and the departed volume $D(\tau)$. Due to the physical GPU capacity, the maximum volume OPT can process per step is bounded by $M$, introducing the differential constraint $D(\tau) - D(\tau-1) \le M$.

At any time $\tau$, the unprocessed volume dictates $A(\tau) \ge D(\tau)$. Plotting these two curves constructs a closed backlogged area. We calculate this $Area(T)$ using two equivalent microscopic and macroscopic perspectives:

\textbf{Macroscopic Vertical Slicing:} Integrating over time yields the total volumetric backlog in $[1, T]$:
\begin{equation}
Area(T) = \sum_{\tau=1}^T (A(\tau) - D(\tau))
\end{equation}

\textbf{Microscopic Horizontal Slicing:} Slicing by requests, each request $i$ contributes a geometric shape consisting of a rectangle (its peak volume waiting in the queue) and a jagged triangle (its dynamic generation footprint). The left boundary is its arrival $a_i$, and the lower boundary spans its pure queuing time $W_i^q$. The volume consumed during its generation phase reduces token-by-token. Its area contribution is:
\begin{equation}
Area_i = W_i^q \cdot vol_i + \sum_{k=1}^{o_i} \left( vol_i - \left(s_i \cdot k + \frac{k+k^2}{2}\right) \right) = W_i^q \cdot vol_i + s_i \frac{o_i^2 - o_i}{2} + \frac{o_i^3 - o_i}{3}
\end{equation}

Let $N(T)$ be the total number of requests that have both arrived and completed within the time interval $[1, T]$. Equating both perspectives gives:
\begin{equation}
\sum_{\tau=1}^T (A(\tau) - D(\tau)) \ge \sum_{i=1}^{N(T)} \left( vol_i \cdot W_i^q + s_i \frac{o_i^2 - o_i}{2} + \frac{o_i^3 - o_i}{3} \right)
\end{equation}

Dividing both sides by $T$ and extracting the arrival rate $\frac{N(T)}{T}$:
\begin{equation}
\frac{1}{T} \sum_{\tau=1}^T (A(\tau) - D(\tau)) \ge \frac{N(T)}{T} \left( \frac{1}{N(T)} \sum_{i=1}^{N(T)} \left( vol_i \cdot W_i^q + s_i \frac{o_i^2 - o_i}{2} + \frac{o_i^3 - o_i}{3} \right) \right)
\end{equation}

Taking the limit $T \to \infty$, the left-hand side converges to the expected residual volume $\mathbb{E}[U]$, and $\frac{N(T)}{T} \to \lambda$. We strictly obtain:
\begin{equation}
\mathbb{E}[U] \ge \lambda \mathbb{E}[vol \cdot W^q] + \lambda \mathbb{E} \left[ s_i \frac{o_i^2 - o_i}{2} + \frac{o_i^3 - o_i}{3} \right]
\end{equation} %到这一步其实都是distribution- free的，不依赖于什么假设

For this derivation, we assume that the prompt length $s$ and the geometric reference output length $o$ are independent, so that mixed moments factorize; in particular, $\mathbb{E}[so] = \mathbb{E}[s]\mathbb{E}[o].$ Let the intrinsic generation area term be $W_{exact} = \lambda \left( \mathbb{E}[s]\frac{1-\mu}{\mu^2} + \frac{2(1-\mu)}{\mu^3} \right) = \lambda \frac{1-\mu}{\mu^2} \left( \mathbb{E}[s] + \frac{2}{\mu} \right)$. Since OPT consumes volume at a maximum rate of $M$, we establish:
\begin{equation}
M \cdot \mathbb{E}[TEL(\text{OPT})] \ge \mathbb{E}[U] \ge W_{exact} \implies \mathbb{E}[TEL(\text{OPT})] \ge \frac{W_{exact}}{M}
\end{equation}
\end{proof}

\subsection{Proof of Theorem \ref{thm:svf_poisson}}

\begingroup
\def\thetheorem{\ref{thm:svf_poisson}}
\begin{theorem}[Stochastic Bound of SVF, Restated]
Under the stochastic-analysis setting described above, with Poisson arrivals of rate $\lambda$ and the peak-memory constraint $s+o \le \alpha M$, the expected-latency ratio for SVF is bounded by:
\begin{equation}
    \frac{ \mathbb E[TEL({\text{SVF}})] }{ \mathbb E[TEL({\mathrm{OPT}})] } \le 1 + \frac{2}{(1-\alpha)(1-\mu)(1-\rho)}.
\end{equation}
\end{theorem}
\addtocounter{theorem}{-1}
\endgroup

\begin{proof}
Before deriving the bound, we define the unfinished volume $U(t) = U_{run}(t) + U_{queue}(t)$ and introduce two fundamental lemmas to constrain the system dynamics.

\textbf{Lemma 1.} \textit{If $U_{queue}(t) > 0$, the expected processing rate in any macroscopic window $\Delta t = \Theta(\mathbb{E}[o])$ containing $t$ strictly satisfies $\bar{r} > (1-\alpha)M/2$.}
\begin{proof}
If the queue is non-empty, the head-of-line request $j$ (peak $p_j$) is blocked by the forward-looking check. Thus, at some future completion time $t_{peak}$, the memory must violate the capacity: $M(t_{peak}) > M - p_j$.
Consider the window $[t, t_{peak}]$ with length $\Delta t$. We strictly partition the active GPU requests into a Finishing Set $F$ (completes before $t_{peak}$) and a Surviving Set $S$ (active at $t_{peak}$).
The volumetric workload consumed exclusively by set $S$ is:
\begin{equation}
V_S = \sum_{\tau=t}^{t_{peak}} M_S(\tau) = \Delta t \cdot M_S(t) + |S|\frac{\Delta t(\Delta t+1)}{2}
\end{equation}
The average rate contributed by $S$ is:
\begin{equation}
\bar{r}_S = \frac{V_S}{\Delta t} = M_S(t) + \frac{|S|(\Delta t+1)}{2} = \frac{M_S(t) + M_S(t_{peak}) + |S|}{2} > \frac{M - p_j}{2}
\end{equation}
Therefore, the total processing rate $\bar{r} = \bar{r}_S + \bar{r}_F > \frac{M - p_j}{2}$. By our safety assumption $s+o \le \alpha M$, we have $p_j \le \alpha M$. 
Thus, $\bar{r} > \frac{M - p_j}{2}\ge \frac{M - \alpha M}{2}=\frac{(1-\alpha)M}{2}$.
\end{proof}

% \textbf{Lemma 2 (Stochastic Memoryless Bound of $U_{run}(t)$).} \textit{At any steady-state moment, the expected residual volume of non-preemptive active requests is a constant $\mathbb{E}[U_{run}] = W_0 = \frac{\lambda \mathbb{E}[vol]}{\mu} + \frac{\lambda}{\mu^3}$.}
% \begin{proof}
% Let $N(t)$ be the active requests with generated tokens $k_i$. The physical memory is $M(t) = \sum_{i=1}^{N(t)} (s_i + k_i) \le M$. Due to the memoryless property of $o \sim \text{Geo}(\mu)$, the remaining length $O_{rem,i}$ follows identical $\text{Geo}(\mu)$. The expected residual volume is:
% \begin{equation}
% \mathbb{E}[U_{run}(t)] = \sum_{i=1}^{N(t)} \mathbb{E} \left[ (s_i+k_i)O_{rem,i} + \frac{O_{rem,i} + O_{rem,i}^2}{2} \right] = \frac{1}{\mu} \sum_{i=1}^{N(t)} (s_i+k_i) + \frac{N(t)}{\mu^2}
% \end{equation}
% In steady state, the time average of memory $\lim_{T \to \infty} \frac{1}{T} \int_0^T M(t)dt = \lambda \mathbb{E}[vol]$. The average token generation rate equals the throughput $N(t)$, so $\mathbb{E}[N(t)] = \lambda \mathbb{E}[o] = \lambda/\mu$. Substituting these yields $W_0$.
% \end{proof}

\textbf{Lemma 2.} \textit{At any steady-state moment, the expected residual volume of non-preemptive active requests is a constant $\mathbb{E}[U_{run}] = W_0 = \frac{\lambda \mathbb{E}[vol]}{\mu} + \frac{\lambda}{\mu^3}=\lambda \left( \frac{\mathbb{E}[s]}{\mu^2} + \frac{2}{\mu^3} \right).$}

\begin{proof}
Let $N(t)$ be the active requests with generated tokens $k_i$. The physical memory is $M(t) = \sum_{i=1}^{N(t)} (s_i + k_i) \le M$. Due to the memoryless property of $o \sim \text{Geo}(\mu)$, the remaining length $O_{rem,i}$ follows identical $\text{Geo}(\mu)$. The expected residual volume is:
\begin{equation}
\mathbb{E}[U_{run}(t)] = \sum_{i=1}^{N(t)} \mathbb{E} \left[ (s_i+k_i)O_{rem,i} + \frac{O_{rem,i} + O_{rem,i}^2}{2} \right] = \frac{1}{\mu} \sum_{i=1}^{N(t)} (s_i+k_i) + \frac{N(t)}{\mu^2}
\end{equation}
In steady state, taking the expectation over the system dynamics, the average memory equals the cumulative volume arrival rate $\mathbb{E}[M(t)] = \lim_{T \to \infty} \frac{1}{T} \int_0^T M(t)dt = \lambda \mathbb{E}[vol]$. Concurrently, the expected active batch size equals the throughput, yielding $\mathbb{E}[N(t)] = \lambda \mathbb{E}[o] = \frac{\lambda}{\mu}$. Substituting these expectations into the residual volume formula gives:
\begin{equation}
\mathbb{E}[U_{run}] = \frac{1}{\mu} \mathbb{E}[M(t)] + \frac{1}{\mu^2} \mathbb{E}[N(t)] = \frac{\lambda \mathbb{E}[vol]}{\mu} + \frac{\lambda}{\mu^3}
\end{equation}
To rigorously align this with the mathematical constant $W_0$, we explicitly expand the expected volume $\mathbb{E}[vol]$. By definition, $vol_i = s_i o_i + \frac{o_i^2 + o_i}{2}$. For $o \sim \text{Geo}(\mu)$, we have $\mathbb{E}[o] = \frac{1}{\mu}$ and $\mathbb{E}[o^2] = \frac{2-\mu}{\mu^2}$. For this derivation, we assume that the prompt length $s$ and the geometric reference output length $o$ are independent, so that mixed moments factorize; in particular, $\mathbb{E}[so] = \mathbb{E}[s]\mathbb{E}[o].$ Assuming independence between prompt and output lengths, taking the expectation yields:
\begin{equation}
\mathbb{E}[vol] = \mathbb{E}[s]\mathbb{E}[o] + \frac{1}{2}(\mathbb{E}[o^2] + \mathbb{E}[o]) = \frac{\mathbb{E}[s]}{\mu} + \frac{1}{2}\left( \frac{2-\mu}{\mu^2} + \frac{\mu}{\mu^2} \right) = \frac{\mathbb{E}[s]}{\mu} + \frac{1}{\mu^2}
\end{equation}
Substituting this exact expansion back into the expectation of $U_{run}$ elegantly recovers the exact form of the theoretical constant $W_0$:
\begin{equation}
\mathbb{E}[U_{run}] = \frac{\lambda}{\mu} \left( \frac{\mathbb{E}[s]}{\mu} + \frac{1}{\mu^2} \right) + \frac{\lambda}{\mu^3} = \lambda \left( \frac{\mathbb{E}[s]}{\mu^2} + \frac{2}{\mu^3} \right) \equiv W_0
\end{equation}
\end{proof}

\textbf{Aligning the Upper Bound:}

For a request with volume $v$, its expected queuing time $\mathbb{E}[W_q(v)]$ is determined by the time required to consume three workloads at the guaranteed rate $\bar{r} \ge (1-\alpha)M/2$: (1) the non-preemptive residual volume $\mathbb{E}[U_{run}] = W_0$, (2) the existing higher-priority queue $\mathbb{E}[U_{queue, \le v}]$, and (3) new high-priority arrivals during its wait. 

To formally quantify $\mathbb{E}[U_{queue, \le v}]$, we consider a sufficiently large macroscopic time window $T$. The queue volume integral $\sum_{\tau=1}^T U_{queue, \le v}(\tau)$ strictly equals the sum of individual waiting rectangles $\sum_{i=1}^{N_{\le v}(T)} vol_i \cdot W_{q,i}$. Dividing both sides by $T$ and taking the limit $T \to \infty$:
\begin{equation}
\lim_{T \to \infty} \frac{1}{T} \sum_{\tau=1}^T U_{queue, \le v}(\tau) = \lim_{T \to \infty} \frac{N_{\le v}(T)}{T} \left( \frac{1}{N_{\le v}(T)} \sum_{i=1}^{N_{\le v}(T)} vol_i \cdot W_{q,i} \right)
\end{equation}
By ergodicity, the left side converges to the spatial expectation $\mathbb{E}[U_{queue, \le v}]$, while $\frac{N_{\le v}(T)}{T} \to \lambda_{\le v}$. Because strict FCFS within identical priorities implies independence between a request's volume and its waiting time within that priority class, we can strictly decompose the expectation: $\mathbb{E}[U_{queue, \le v}] = \lambda_{\le v} \mathbb{E}[vol \cdot W_q \mid vol \le v] = \sum_{x=1}^v \lambda_x \cdot x \cdot \mathbb{E}[W_q(x)]$.

Substituting this into the expected balance equation yields:
\begin{equation}
\frac{(1-\alpha)M}{2} \mathbb{E}[W_q(v)] \le W_0 + \sum_{x=1}^v \lambda_x \cdot x \cdot \mathbb{E}[W_q(x)] + \left( \sum_{x=1}^{v-1} \lambda_x \cdot x \right) \mathbb{E}[W_q(v)]
\end{equation}

To solve this, let $D(v) = \frac{(1-\alpha)M}{2} - \sum_{x=1}^{v-1} \lambda_x \cdot x$. Notice that $D(v) - D(v+1) = \lambda_v \cdot v$. Rearranging the inequality, we obtain the bounding relation \textbf{(A)}:
\begin{equation} \label{eq:bound_A}
\mathbb{E}[W_q(v)] \cdot D(v+1) \le W_0 + \sum_{x=1}^{v-1} \lambda_x \cdot x \cdot \mathbb{E}[W_q(x)]
\end{equation}

We construct an exact auxiliary sequence $W(v)$ governed by equality: $W(v)D(v) = W_0 + \sum_{x=1}^v \lambda_x \cdot x \cdot W(v)$. Rearranging this analogously yields \textbf{(B)}:
\begin{equation} \label{eq:bound_B}
W(v) \cdot D(v+1) = W_0 + \sum_{x=1}^{v-1} \lambda_x \cdot x \cdot W(x)
\end{equation}

We deploy \textbf{Mathematical Induction} to prove the global dominance $\mathbb{E}[W_q(v)] \le W(v)$ for all $v$:
\begin{itemize}
    \item \textit{Base Case ($v=1$):} Inequality \eqref{eq:bound_A} gives $\mathbb{E}[W_q(1)]D(2) \le W_0$, and equality \eqref{eq:bound_B} gives $W(1)D(2) = W_0$. Since system stability ensures $D(2) > 0$, it strictly holds that $\mathbb{E}[W_q(1)] \le \frac{W_0}{D(2)} = W(1)$.
    \item \textit{Inductive Step:} Assume $\mathbb{E}[W_q(x)] \le W(x)$ holds for all $x \in \{1, \dots, v-1\}$. Evaluating inequality \eqref{eq:bound_A} for $v$:
    \begin{equation}
    \mathbb{E}[W_q(v)] \cdot D(v+1) \le W_0 + \sum_{x=1}^{v-1} \lambda_x \cdot x \cdot \mathbb{E}[W_q(x)] \le W_0 + \sum_{x=1}^{v-1} \lambda_x \cdot x \cdot W(x)
    \end{equation}
    The rightmost expression exactly matches the right side of \eqref{eq:bound_B}. Thus, $\mathbb{E}[W_q(v)] \cdot D(v+1) \le W(v) \cdot D(v+1)$. Since $D(v+1) > 0$, we conclude $\mathbb{E}[W_q(v)] \le W(v)$.
\end{itemize}

To derive a closed form for $W(v)$, we subtract the $v-1$ equality from the $v$ equality:
\begin{equation}
W(v)D(v) - W(v-1)D(v-1) = \lambda_v \cdot v \cdot W(v) = W(v)[D(v) - D(v+1)]
\end{equation}
Canceling $W(v)D(v)$ from both sides yields the recurrence $W(v-1)D(v-1) = W(v)D(v+1)$. Rewriting this as $W(v) = W(v-1) \frac{D(v-1)}{D(v+1)}$ and unrolling the telescoping product down to $v=1$ (where $W(1) = W_0/D(2)$):
\begin{equation}
W(v) = W(1) \frac{D(1)D(2)}{D(v)D(v+1)} = \frac{D(1)W_0}{D(v)D(v+1)}
\end{equation}

Next, we calculate the global expected queuing volume $\lambda \mathbb{E}[vol \cdot W_q] \le \sum_{v=1}^\infty \lambda_v \cdot v \cdot W(v)$. Let $S(v-1) = \sum_{x=1}^{v-1} \lambda_x \cdot x \cdot W(x)$. We know $S(v) - S(v-1) = \lambda_v \cdot v \cdot W(v)$. Summing this over the infinite domain gives:
\begin{equation}
\sum_{v=1}^\infty \left( S(v) - S(v-1) \right) = S(\infty) - S(0)
\end{equation}
By definition, $S(0) = 0$, and from the auxiliary equality, $S(\infty) = W(\infty)D(\infty+1) - W_0$. Substituting the unrolled $W(\infty)$ expression yields $S(\infty) = W_0 \left( \frac{D(1)}{D(\infty)} - 1 \right)$.
Given boundary conditions $D(1) = (1-\alpha)M/2$ and $D(\infty) = (1-\alpha)M/2 - \lambda \mathbb{E}[vol]$, we evaluate the limit:
\begin{equation}
\lambda \mathbb{E}[vol \cdot W_q] \le S(\infty) = W_0 \frac{\lambda \mathbb{E}[vol]}{(1-\alpha)M/2 - \lambda \mathbb{E}[vol]} \implies \mathbb{E}[vol \cdot W_q] \le W_0 \frac{\mathbb{E}[vol]}{(1-\alpha)M/2 - \lambda \mathbb{E}[vol]}
\end{equation}

Because volume $vol$ and queuing delay $W_q$ are both monotonically increasing, identically ordered sequences, we apply the \textbf{Harris Inequality} to strictly separate the correlated variables: $\mathbb{E}[vol \cdot W_q] \ge \mathbb{E}[vol] \cdot \mathbb{E}[W_q]$.
Substituting this establishes the strict queuing bound:
\begin{equation}
\mathbb{E}[W_q] \le \frac{W_0}{(1-\alpha)M/2 - \lambda \mathbb{E}[vol]}
\end{equation}

Finally, recall that $W_0 = \lambda\left(\frac{\mathbb{E}[s]}{\mu^2} + \frac{2}{\mu^3}\right)$ and the exact intrinsic generation area is $W_{exact} = \lambda(1-\mu)\left(\frac{\mathbb{E}[s]}{\mu^2} + \frac{2}{\mu^3}\right)$. This establishes the elegant mapping $W_0 = \frac{W_{exact}}{1-\mu}$.
Aligning the end-to-end latency of our algorithm (SVF) with the offline optimal (OPT):
\begin{equation}
\begin{aligned}
\mathbb{E}[TEL(\text{SVF})]
&\le \mathbb{E}[o]
+ \frac{W_0}{(1-\alpha)M/2 - \lambda \mathbb{E}[vol]} \\
&\le \mathbb{E}[TEL(\text{OPT})]
+ \frac{1}{1-\mu}
\frac{W_{exact}}{M}
\frac{M}{(1-\alpha)M/2 - \lambda \mathbb{E}[vol]}.
\end{aligned}
\end{equation}
By factoring out $\mathbb{E}[TEL(\text{OPT})] \ge \frac{W_{exact}}{M}$ and defining the worst-case system utilization factor $\rho = \frac{\lambda \mathbb{E}[vol]}{(1-\alpha)M/2}$, we elegantly conclude the expected-latency ratio bound:
\begin{equation}
\frac{ \mathbb E[TEL({\text{SVF}})] }{ \mathbb E[TEL({\mathrm{OPT}})] }  \le 1 + \frac{2}{(1-\alpha)(1-\mu)(1-\rho)}
\end{equation}
\end{proof}

\subsection{Proof of Theorem \ref{thm:svf_1bit_poisson}}

\begingroup
\def\thetheorem{\ref{thm:svf_1bit_poisson}}
\begin{theorem}[Stochastic Bound of 1-bit SVF, Restated]
Under the same stochastic-analysis setting described above, consider the 1-bit SVF algorithm with threshold $\theta$ and proxy lengths given by their conditional expectations under the geometric reference model,
specifically, $O_0 = \frac{1}{\mu} - \frac{\theta(1-\mu)^\theta}{1 - (1-\mu)^\theta}$ for short requests ($o \le \theta$) and $O_1 = \theta + \frac{1}{\mu}$ for long requests ($o > \theta$). The expected-latency ratio is bounded by:
\begin{equation}
    \frac{ \mathbb E[TEL({1\text{-bit SVF}})] }{ \mathbb E[TEL({\mathrm{OPT}})] } \le 1 + \left( \frac{\mathbb{E}[vol]}{\mathbb{E}[vol] - \epsilon} \right) \frac{2}{(1-\alpha)(1-\mu)(1-\rho)}
\end{equation}
where the volume distortion penalty is $\epsilon = \frac{1}{2} \left( \frac{\theta (1-\mu)^\theta}{1 - (1-\mu)^\theta} \right)^2$.
\end{theorem}
\addtocounter{theorem}{-1}
\endgroup

\begin{proof}
By categorizing requests into discrete classes based on the proxy volume $\hat{v}$, the system essentially operates as a non-preemptive discrete priority queue with $K$ priorities. 
Let $v_k = \mathbb{E}[vol_{true} \mid \hat{v} = \hat{v}_k]$ be the true expected volume of priority class $k$, and $\lambda_k$ be its arrival rate. We define the cumulative load ratio of the first $k$ priorities as $R_k = \sum_{j=1}^k \lambda_j v_j$. Clearly, $R_0 = 0$ and $R_K = \lambda \mathbb{E}[vol]$.

For any request in priority class $k$, its expected queuing time $W_k$ must wait for three components to be processed at the guaranteed rate $\bar{r} \ge (1-\alpha)M/2$:
\begin{itemize}
    \item The non-preemptive residual volume: $W_0$ (Lemma 2 holds).
    \item Existing higher or equal priority requests in the queue: $\mathbb{E}[U_{queue, \le k}] = \sum_{j=1}^k \lambda_j v_j W_j$.
    \item New higher priority requests arriving during its wait (strictly less than $k$): $W_k \sum_{j=1}^{k-1} \lambda_j v_j$.
\end{itemize}

Summing these gives the expected balance equation:
\begin{equation}
    W_k = \frac{W_0 + \sum_{j=1}^k \lambda_j v_j W_j + W_k \sum_{j=1}^{k-1} \lambda_j v_j}{\bar{r}}
\end{equation}
We introduce an upper bound sequence $W_k \le W_k'$, allowing us to safely replace $W_k$ and $\bar{r}$ with their bounds:
\begin{equation}
    W_k \le \frac{W_0 + \sum_{j=1}^{k-1} \lambda_j v_j W_j + \lambda_k v_k W_k' + W_k' \sum_{j=1}^{k-1} \lambda_j v_j}{(1-\alpha)M/2} = W_k'
\end{equation}
Multiplying by $(1-\alpha)M/2$, we obtain the discrete algebraic equation:
\begin{equation}
    W_k' \cdot (1-\alpha)M/2 = W_0 + \sum_{j=1}^{k-1} \lambda_j v_j W_j + W_k' \sum_{j=1}^k \lambda_j v_j
\end{equation}
Substituting $\lambda_j v_j = \rho_j = R_j - R_{j-1}$ and rearranging the $W_k'$ terms to the left:
\begin{equation}
    W_k' \left( \frac{(1-\alpha)M}{2} - R_k \right) = W_0 + \sum_{j=1}^{k-1} \rho_j W_j
\end{equation}
Evaluating this for class $k-1$ yields:
\begin{equation}
    W_{k-1}' \left( \frac{(1-\alpha)M}{2} - R_{k-1} \right) = W_0 + \sum_{j=1}^{k-2} \rho_j W_j
\end{equation}
Subtracting the $k-1$ equation from the $k$ equation isolates the difference:
\begin{equation}
    W_k' \left( \frac{(1-\alpha)M}{2} - R_k \right) - W_{k-1}' \left( \frac{(1-\alpha)M}{2} - R_{k-1} \right) = \rho_{k-1} W_{k-1} = (R_{k-1} - R_{k-2}) W_{k-1}
\end{equation}
Since $W_{k-1} \le W_{k-1}'$, we can bound the right side:
\begin{equation}
\begin{aligned}
W_k' \left( \frac{(1-\alpha)M}{2} - R_k \right)
&\le W_{k-1}' \left( \frac{(1-\alpha)M}{2} - R_{k-1} + R_{k-1} - R_{k-2} \right) \\
&= W_{k-1}' \left( \frac{(1-\alpha)M}{2} - R_{k-2} \right).
\end{aligned}
\end{equation}
\begin{equation}
    \implies W_k' \le W_{k-1}' \frac{\frac{(1-\alpha)M}{2} - R_{k-2}}{\frac{(1-\alpha)M}{2} - R_k}
\end{equation}
Starting from the base case $W_1' = \frac{W_0}{\frac{(1-\alpha)M}{2} - R_1}$, we unroll this sequence via telescoping multiplication:
% \begin{equation}
%     W_k' \le \frac{W_0}{\frac{(1-\alpha)M}{2} - R_1} \cdot \frac{\frac{(1-\alpha)M}{2} - R_0}{\frac{(1-\alpha)M}{2} - R_2} \cdot \frac{\frac{(1-\alpha)M}{2} - R_1}{\frac{(1-\alpha)M}{2} - R_3} \cdots \frac{\frac{(1-\alpha)M}{2}- R_{k-2}}{\frac{(1-\alpha)M}{2} - R_k} = \frac{W_0 \frac{(1-\alpha)M}{2}}{\left(\frac{(1-\alpha)M}{2} - R_{k-1}\right)\left(\frac{(1-\alpha)M}{2} - R_k\right)}
% \end{equation}
\begin{equation}
\begin{aligned}
W_k' &\le \frac{W_0}{\frac{(1-\alpha)M}{2} - R_1} \cdot \frac{\frac{(1-\alpha)M}{2} - R_0}{\frac{(1-\alpha)M}{2} - R_2} \cdot \frac{\frac{(1-\alpha)M}{2} - R_1}{\frac{(1-\alpha)M}{2} - R_3} \cdots \frac{\frac{(1-\alpha)M}{2} - R_{k-2}}{\frac{(1-\alpha)M}{2} - R_k}\\
&= \frac{ W_0 \frac{(1-\alpha)M}{2} }{ \left(\frac{(1-\alpha)M}{2} - R_{k-1}\right) \left(\frac{(1-\alpha)M}{2} - R_k\right) }.
\end{aligned}
\end{equation}

Direct summation of $W_k'$ is intractable, so we solve for the system's expected queuing workload $\lambda \mathbb{E}[vol \cdot W_q]$. By the Law of total expectation:
\begin{equation}
    \lambda \mathbb{E}[vol \cdot W_q] = \sum_{k=1}^K \lambda_k v_k W_k \le \sum_{k=1}^K \lambda_k v_k W_k' = \sum_{k=1}^K (R_k - R_{k-1}) \frac{W_0 \frac{(1-\alpha)M}{2}}{\left(\frac{(1-\alpha)M}{2} - R_{k-1}\right)\left(\frac{(1-\alpha)M}{2} - R_k\right)}
\end{equation}
Through algebraic telescoping cancellation, the fraction splits perfectly:
% \begin{equation}
%     \sum_{k=1}^K \frac{R_k - R_{k-1}}{\left(\frac{(1-\alpha)M}{2} - R_{k-1}\right)\left(\frac{(1-\alpha)M}{2} - R_k\right)} = \sum_{k=1}^K \left( \frac{1}{\frac{(1-\alpha)M}{2} - R_k} - \frac{1}{\frac{(1-\alpha)M}{2} - R_{k-1}} \right) \\ = \frac{1}{\frac{(1-\alpha)M}{2} - R_K} - \frac{1}{\frac{(1-\alpha)M}{2}- R_0}
% \end{equation}
\begin{equation}
\begin{aligned}
    \sum_{k=1}^K 
    \frac{R_k - R_{k-1}}
    {\left(\frac{(1-\alpha)M}{2} - R_{k-1}\right)
     \left(\frac{(1-\alpha)M}{2} - R_k\right)}
    &=
    \sum_{k=1}^K
    \left(
        \frac{1}{\frac{(1-\alpha)M}{2} - R_k}
        -
        \frac{1}{\frac{(1-\alpha)M}{2} - R_{k-1}}
    \right)
    \\
    &=
    \frac{1}{\frac{(1-\alpha)M}{2} - R_K}
    -
    \frac{1}{\frac{(1-\alpha)M}{2} - R_0}.
\end{aligned}
\end{equation}
Substituting the boundary conditions $R_0 = 0$ and $R_K = \lambda \mathbb{E}[vol]$:
\begin{equation}
    \lambda \mathbb{E}[vol \cdot W_q] \le W_0 \frac{(1-\alpha)M}{2} \left( \frac{1}{\frac{(1-\alpha)M}{2} - \lambda \mathbb{E}[vol]} - \frac{(1-\alpha)M}{2} \right) = W_0 \frac{\lambda \mathbb{E}[vol]}{\frac{(1-\alpha)M}{2} - \lambda \mathbb{E}[vol]}
\end{equation}
Canceling $\lambda$ from both sides gives the expected queuing volume bound:
\begin{equation} \label{eq:1bit_vol_wq}
    \mathbb{E}[vol \cdot W_q] \le W_0 \frac{\mathbb{E}[vol]}{\frac{(1-\alpha)M}{2} - \lambda \mathbb{E}[vol]}
\end{equation}

To utilize \eqref{eq:1bit_vol_wq}, we must replace the true volume $vol_{true}$ with the algorithm-perceived proxy. For any class $m$, the proxy length is $O_m$. 

For this derivation, we assume that the prompt length $s$ and the geometric reference output length $o$ are independent, so that mixed moments factorize; in particular, $\mathbb{E}[so] = \mathbb{E}[s]\mathbb{E}[o].$ The true expected volume is:
\begin{equation}
    v_k = \mathbb{E}[vol_{true} \mid s, class~m] = s \cdot \mathbb{E}[o \mid class~m] + \mathbb{E} \left[ \frac{o^2+o}{2} \Biggm| class~m \right]
\end{equation}
The proxy volume is defined as $\hat{v}_k = s \cdot O_m + \frac{O_m^2 + O_m}{2}$. We define the local volume perturbation introduced by classification as:
\begin{equation}
    \Delta_m = \mathbb{E} \left[ \frac{o^2+o}{2} \Biggm| class~m \right] - \frac{O_m^2 + O_m}{2}
\end{equation}
This ensures $v_k = \hat{v}_k + \Delta_{m(k)}$. Substituting this back into the expectation:
\begin{equation}
    \mathbb{E}[vol \cdot W_q] = \mathbb{E}[(\hat{v} + \Delta(\hat{v})) W_q] = \mathbb{E}[\hat{v} \cdot W_q] + \mathbb{E}[\Delta(\hat{v}) \cdot W_q] \le W_0 \frac{\mathbb{E}[vol]}{\frac{(1-\alpha)M}{2} - \lambda \mathbb{E}[vol]}
\end{equation}
Since the 1-bit SVF schedules strictly in ascending order of $\hat{v}$, $\hat{v}_k$ and $W_k$ are identically sorted. By the discrete Chebyshev sum inequality:
\begin{equation}
    \mathbb{E}[\hat{v} \cdot W_q] = \sum p_k \hat{v}_k W_k \ge \mathbb{E}[\hat{v}] \mathbb{E}[W_q]
\end{equation}
Let $\Delta_{min} = \min(\Delta_0, \Delta_1)$. Since $W_q \ge 0$, we safely bound $\mathbb{E}[\Delta(\hat{v}) W_q] \ge \Delta_{min} \mathbb{E}[W_q]$. Factoring out $\mathbb{E}[W_q]$:
\begin{equation}
    (\mathbb{E}[\hat{v}] + \Delta_{min}) \mathbb{E}[W_q] = (\mathbb{E}[vol] - \mathbb{E}[\Delta(\hat{v})] + \Delta_{min}) \mathbb{E}[W_q] \le W_0 \frac{\mathbb{E}[vol]}{\frac{(1-\alpha)M}{2}- \lambda \mathbb{E}[vol]}
\end{equation}
Defining the systemic volume distortion penalty as $\epsilon = \mathbb{E}[\Delta(\hat{v})] - \Delta_{min} \ge 0$, we isolate $\mathbb{E}[W_q]$:
\begin{equation} \label{eq:1bit_wq_final}
    \mathbb{E}[W_q] \le \frac{\mathbb{E}[vol]}{\mathbb{E}[vol] - \epsilon} \frac{W_0}{\frac{(1-\alpha)M}{2} - \lambda \mathbb{E}[vol]}
\end{equation}

We now rigorously derive the exponential decay form of $\epsilon$. By the definition of variance, $\Delta_m = \frac{1}{2} ( \mathbb{E}[o^2|m] - O_m^2 ) + \frac{1}{2} ( \mathbb{E}[o|m] - O_m ) = \frac{1}{2}\text{Var}(o \mid m)$.
Let $q = 1-\mu$. The tail probabilities for the threshold $\theta$ are $p_1 = \mathbb{P}(o > \theta) = q^\theta$, and $p_0 = 1 - q^\theta$. 
Due to the memoryless property of the Geometric distribution, the variance of the long class ($o > \theta$) is identical to the unconditioned variance: $\text{Var}(o \mid o > \theta) = \frac{1-\mu}{\mu^2}$. Thus, $\Delta_1 = \frac{1-\mu}{2\mu^2}$.

To find $\Delta_0$, we first compute the second moment of the proxy lengths $\mathbb{E}[O_m^2] = p_0 O_0^2 + p_1 O_1^2$. From the law of total expectation, $p_0 O_0 + p_1 O_1 = \frac{1}{\mu} \implies p_0 O_0 = \frac{1}{\mu} - p_1 O_1$.
Substituting this constraint:
\begin{equation}
    p_0 O_0^2 + p_1 O_1^2 = \frac{(\frac{1}{\mu} - p_1 O_1)^2}{p_0} + p_1 O_1^2 = \frac{\frac{1}{\mu^2} - \frac{2 p_1 O_1}{\mu} + p_1^2 O_1^2 + p_0 p_1 O_1^2}{p_0} = \frac{\frac{1}{\mu^2} - p_1\left( \frac{2O_1}{\mu} - O_1^2 \right)}{p_0}
\end{equation}
Because the long proxy is $O_1 = \theta + \frac{1}{\mu}$, its internal term simplifies beautifully: $\frac{2O_1}{\mu} - O_1^2 = \frac{1}{\mu^2} - \theta^2$. Substituting this back:
\begin{equation}
    p_0 O_0^2 + p_1 O_1^2 = \frac{\frac{1}{\mu^2} - p_1(\frac{1}{\mu^2} - \theta^2)}{p_0} = \frac{1}{\mu^2} + \frac{p_1 \theta^2}{p_0}
\end{equation}

Now, we evaluate the global expectation of the perturbation:
\begin{equation}
    \mathbb{E}[\Delta(\hat{v})] = \mathbb{E}\left[\frac{o^2+o}{2}\right] - \mathbb{E}\left[\frac{O_m^2+O_m}{2}\right] = \frac{1}{\mu^2} - \frac{1}{2}(\mathbb{E}[O_m^2] + \mathbb{E}[O_m])
\end{equation}
\begin{equation}
    \mathbb{E}[\Delta(\hat{v})] = \frac{1}{\mu^2} - \frac{1}{2} \left( \frac{1}{\mu^2} + \frac{p_1 \theta^2}{p_0} + \frac{1}{\mu} \right) = \frac{1-\mu}{2\mu^2} - \frac{p_1 \theta^2}{2p_0} = \Delta_1 - \frac{p_1 \theta^2}{2p_0}
\end{equation}
Since $\mathbb{E}[\Delta(\hat{v})] < \Delta_1$, and we know by definition $\mathbb{E}[\Delta(\hat{v})] = p_0 \Delta_0 + p_1 \Delta_1$, it strictly necessitates that $\Delta_0 < \Delta_1$. Therefore, $\Delta_{min} = \Delta_0$.
The penalty simplifies to:
\begin{equation}
    \epsilon = \mathbb{E}[\Delta(\hat{v})] - \Delta_0 = p_0 \Delta_0 + p_1 \Delta_1 - \Delta_0 = p_1 (\Delta_1 - \Delta_0)
\end{equation}
From $\mathbb{E}[\Delta(\hat{v})] = \Delta_1 - \frac{p_1 \theta^2}{2p_0}$, we also have $\Delta_1 - \Delta_0 = \frac{p_1 \theta^2}{2p_0^2}$. Substituting this into $\epsilon$:
\begin{equation}
    \epsilon = p_1 \left( \frac{p_1 \theta^2}{2p_0^2} \right) = \frac{1}{2} \left( \frac{p_1 \theta}{p_0} \right)^2 = \frac{1}{2} \left( \frac{\theta q^\theta}{1 - q^\theta} \right)^2
\end{equation}

Finally, substituting $\epsilon$ back into the latency alignment equation \eqref{eq:1bit_wq_final}:
\begin{equation}
    \mathbb{E}[TEL(\text{1-bit SVF})] \le \mathbb{E}[o] + \frac{\mathbb{E}[vol]}{\mathbb{E}[vol] - \epsilon} \frac{W_0}{\frac{(1-\alpha)M}{2} - \lambda \mathbb{E}[vol]}
\end{equation}
Using the established OPT bound $\mathbb{E}[TEL(\text{OPT})] \ge \frac{W_{exact}}{M}$ where $W_{exact} = (1-\mu) W_0$:
\begin{equation}
    \mathbb{E}[TEL(\text{1-bit SVF})] \le \mathbb{E}[TEL(\text{OPT})] + \frac{\mathbb{E}[vol]}{\mathbb{E}[vol] - \epsilon} \frac{1}{1-\mu} \frac{W_{exact}}{M} \frac{M}{\frac{(1-\alpha)M}{2} - \lambda \mathbb{E}[vol]}
\end{equation}
Factoring out $\mathbb{E}[TEL(\text{OPT})]$ and defining $\rho = \frac{\lambda \mathbb{E}[vol]}{(1-\alpha)M/2} \in [0, 1)$, we elegantly conclude:
\begin{equation}
    \frac{ \mathbb E[TEL({1\text{-bit SVF}})] }{ \mathbb E[TEL({\mathrm{OPT}})] } \le 1 + \frac{\mathbb{E}[vol]}{\mathbb{E}[vol] - \epsilon} \frac{2}{(1-\alpha)(1-\mu)(1-\rho)}
\end{equation}
\end{proof}

% \section{Predictor Performance}
% \label{sec:appendix_predictor}
% \begin{table}[t]
% \centering
% \small
% \caption{Predictor accuracy (regression)}
% \label{tab:predictor_regression}
% \begin{tabular}{l r r r r r}
% \toprule
% \textbf{Predictor} & MAE & Median AE & Acc@10 (%) & Acc@50 (%) & Acc@100 (%) \\
% \midrule
% llama-3-8b / lmsys & 205.60 & 91.45 & 18.92 & 39.44 & 51.90 \\
% llama-3-8b / longbench & 19.25 & 4.10 & 70.22 & 94.67 & 96.78 \\
% \bottomrule
% \end{tabular}
% \end{table}

% \begin{table}[t]
% \centering
% \small
% \caption{Predictor accuracy (classifier)}
% \label{tab:predictor_classifier}
% \begin{tabular}{l r r r}
% \toprule
% \textbf{Predictor} & Accuracy (%) & Precision (long) & Recall (long) \\
% \midrule
% lmsys / burst (thr=64) & 79.03 & 0.8569 & 0.8200 \\
% longbench / burst (thr=64) & 93.78 & 0.0000 & 0.0000 \\
% lmsys / poisson (thr=777) & 82.73 & 0.3588 & 0.6026 \\
% longbench / poisson (thr=36) & 77.33 & 0.3022 & 0.8265 \\
% \bottomrule
% \end{tabular}
% \end{table}

%%%%%%%%%%%%%%%%%%%%%%%%%%%%%%%%%%%%%%%%%%%%%%%%%%%%%%%%%%%%

% \newpage
% \input{checklist.tex}

\end{document}